%% file: example_paper.tex
\documentclass{article}

\usepackage[ruled]{algorithm2e}
\usepackage{microtype}
\usepackage{graphicx}
\usepackage{subcaption}
\usepackage{booktabs} % for professional tables

\usepackage{hyperref}

\usepackage{amssymb}
\usepackage{amsthm}
\usepackage{amsmath}
\usepackage{mathtools}
\usepackage{enumitem}
\usepackage{multirow}

\input{commands}

\usepackage[accepted]{icml2021}

\renewcommand{\algorithmicrequire}{\textbf{Input:}}
\renewcommand{\algorithmicensure}{\textbf{Output:}}

\newtheorem*{thm}{\textbf{Theorem}}
\newtheorem{theorem}{\textbf{Theorem}}

\newtheorem{definition}{\textbf{Definition}}

\newcommand{\bX}{\mathbf{X}}
\newcommand{\bY}{\mathbf{Y}}

\newcommand{\bm}{\mathbf{m}}

\newcommand{\bx}{\mathbf{x}}
\newcommand{\by}{\mathbf{y}}
\newcommand{\bR}{\mathbf{R}}

\newcommand{\bv}{\mathbf{v}}

\newcommand{\bA}{\mathbf{A}}
\newcommand{\bB}{\mathbf{B}}

\newcommand{\bU}{\mathbf{U}}

\newcommand{\bS}{\mathbf{S}}
\newcommand{\bQ}{\mathbf{Q}}
\newcommand{\bP}{\mathbf{P}}
\newcommand{\bE}{\mathbf{E}}
\newcommand{\bL}{\mathbf{L}}

\newcommand{\bI}{\mathbf{I}}
\newcommand{\bT}{\mathbf{T}}

\newcommand{\bM}{\mathbf{M}}

\newcommand{\bD}{\mathbf{D}}

\newcommand{\bJ}{\mathbf{J}}

\newlength{\Oldarrayrulewidth}

\icmltitlerunning{Skew Orthogonal Convolutions}

\begin{document}

\twocolumn[
\icmltitle{Skew Orthogonal Convolutions}



\icmlsetsymbol{equal}{*}

\begin{icmlauthorlist}
\icmlauthor{Sahil Singla}{umd}
\icmlauthor{Soheil Feizi}{umd}
\end{icmlauthorlist}

\icmlaffiliation{umd}{Department of Computer Science, University of Maryland, College Park}

\icmlcorrespondingauthor{Sahil Singla}{ssingla@umd.edu}
% \icmlcorrespondingauthor{Soheil Feizi}{sfeizi@umd.edu}

\icmlkeywords{Machine Learning, ICML}

\vskip 0.3in
]



\printAffiliationsAndNotice{}  % leave blank if no need to mention equal contribution

\begin{abstract}
Training convolutional neural networks with a Lipschitz constraint under the $l_{2}$ norm is useful for provable adversarial robustness, interpretable gradients, stable training, etc. While 1-Lipschitz networks can be designed by imposing a 1-Lipschitz constraint on each layer, training such networks requires each layer to be gradient norm preserving (GNP) to prevent gradients from vanishing. However, existing GNP convolutions suffer from slow training, lead to significant reduction in accuracy and provide no guarantees on their approximations. In this work, we propose a GNP convolution layer called \methodnamebold\ (\methodabv) that uses the following mathematical property: when a matrix is {\it Skew-Symmetric}, its exponential function is an {\it orthogonal} matrix. To use this property, we first construct a convolution filter whose Jacobian is Skew-Symmetric. Then, we use the Taylor series expansion of the Jacobian exponential to construct the \methodabv\ layer that is orthogonal. To efficiently implement \methodabv, we keep a finite number of terms from the Taylor series and provide a provable guarantee on the approximation error. Our experiments on CIFAR-10 and CIFAR-100 show that \methodabv\ allows us to train provably Lipschitz, large convolutional neural networks significantly faster than prior works while achieving significant improvements for both standard and certified robust accuracies.

%Then, we use finite number of terms in the Taylor series expansion of the exponential function of the Jacobian to construct the \methodabv\ layer that is approximately Orthogonal with a provable guarantee on the approximation error. Our experiments on CIFAR-10 and CIFAR-100 show that SOC allows us to train provably Lipschitz, large convolutional neural networks significantly faster than the prior works while achieving significant improvements for both standard and certified robust accuracies. 

%$\exp(\bJ)$ and spectral normalization of $\bJ$ to construct the \methodabv\ layer that is approximately Orthogonal with a provable guarantee on the error. Our experiments on CIFAR-10 and CIFAR-100 show that SOC allows us to train provably Lipschitz, large convolutional neural networks significantly faster than the prior works while achieving significant improvements for both standard and certified robust accuracy. 

%
%To train provably Lipschitz convolutional neural networks, \citet{li2019lconvnet} proposed a parametrization of Orthogonal convolution layers called \textbf{B}lock \textbf{C}onvolutional \textbf{O}rthogonal \textbf{P}arametrization (BCOP). However, BCOP suffers from slow training, significant reduction in the accuracy, requires $4$ times more parameters to remedy the disconnectedness issues and provides no guarantees on its approximation. 

%While approaches such as adversarial training have seen much practical success for empirical robustness, it is challenging to achieve a similar practical performance for provable robustness against adversarial examples.
\end{abstract}

\section{Introduction}\label{sec:introduction}
The Lipschitz constant\footnotemark[2] of a neural network puts an upper bound on how much the output is allowed to change in proportion to a change in input. Previous work has shown that a small Lipschitz constant leads to improved generalization bounds \cite{Bartlettgeneralization, long2019sizefree}, adversarial robustness \cite{cisseparseval2017, 42503} and interpretable gradients \cite{tsipras2019robustness}. The Lipschitz constant also upper bounds the change in gradient norm during backpropagation and can thus prevent gradient explosion during training, allowing us to train very deep networks \cite{xiao2018dynamical}. Moreover, the Wasserstein distance between two probability distributions can be expressed as a maximization over 1-Lipschitz functions \cite{villani2012optimal, peyre2018computational}, and has been used for training Wasserstein GANs \cite{arjovsky2017wasserstein, gulrajani2017improved} and Wasserstein VAEs \cite{tolstikhin2018wasserstein}. 

Using the Lipschitz\footnotetext[2]{Unless specified, we use Lipschitz constant under the $l_{2}$ norm.} composition property \big(i.e. $Lip(f \circ g) \leq Lip(f) Lip(g)$\big), a Lipschitz constant of a neural network can be bounded by the product of the Lipschitz constant of all layers. 1-Lipschitz neural networks can thus be designed by imposing a 1-Lipschitz constraint on each layer.  However, \citet{Anil2018SortingOL} identified a key difficulty with this approach: because a layer with a Lipschitz bound of 1 can only reduce the norm of the gradient during backpropagation, each step of backprop gradually attenuates the gradient norm, resulting in a much smaller gradient for the layers closer to the input, thereby making training slow and difficult. To address this problem, they introduced Gradient Norm Preserving (GNP) architectures where each layer preserves the gradient norm by ensuring that the Jacobian of each layer is an {\it Orthogonal} matrix (for all inputs to the layer). For convolutional layers, this involves constraining the Jacobian of each convolution layer to be an Orthogonal matrix \cite{li2019lconvnet, xiao2018dynamical} and using a GNP activation function called GroupSort \cite{Anil2018SortingOL}. 

% \begin{figure}[t]
% \centering
% \begin{subfigure}{0.45\linewidth}
% \centering
% \includegraphics[trim=0cm 22cm 14cm 0cm, clip, width=\linewidth]{figures/skew_filter_color2d}
% \caption{Skew-symmetric filter}
% \label{subfig:skew2d}
% \end{subfigure}\qquad
% \begin{subfigure}{0.45\linewidth}
% \centering
% \includegraphics[trim=0cm 16cm 25cm 0cm, clip, width=\linewidth]{figures/flat_spectral_norm}
% % \caption{\methodabv\ applied to input $\bX$}
% \caption{Spectral normalization}
% \label{subfig:spectral_norm}
% \end{subfigure}
% \caption{Each color denotes a scalar, minus sign $(-)$ on top of some color denotes the negative of the scalar with that color. Given any convolution filter $\bM$, we can construct a Skew-Symmetric filter (Figure \ref{subfig:skew2d}). Next, we apply spectral normalization to bound the norm of jacobian (Figure \ref{subfig:spectral_norm}). On input $\bX$, applying convolution exponential $(\bL \star_{e} \bX)$ results in an Orthogonal convolution (Figure \ref{subfig:conv2d_trunc}). }
% %The notation, $\bL \star^{n} \bX = \bL \star^{n-1}\left(\bL \star \bX\right)$}
% \label{fig:sop_demo}
% \end{figure}

\begin{figure*}[t]
\centering
\begin{subfigure}{0.26\linewidth}
\centering
\includegraphics[trim=7cm 7cm 7cm 5cm, clip, width=\linewidth]{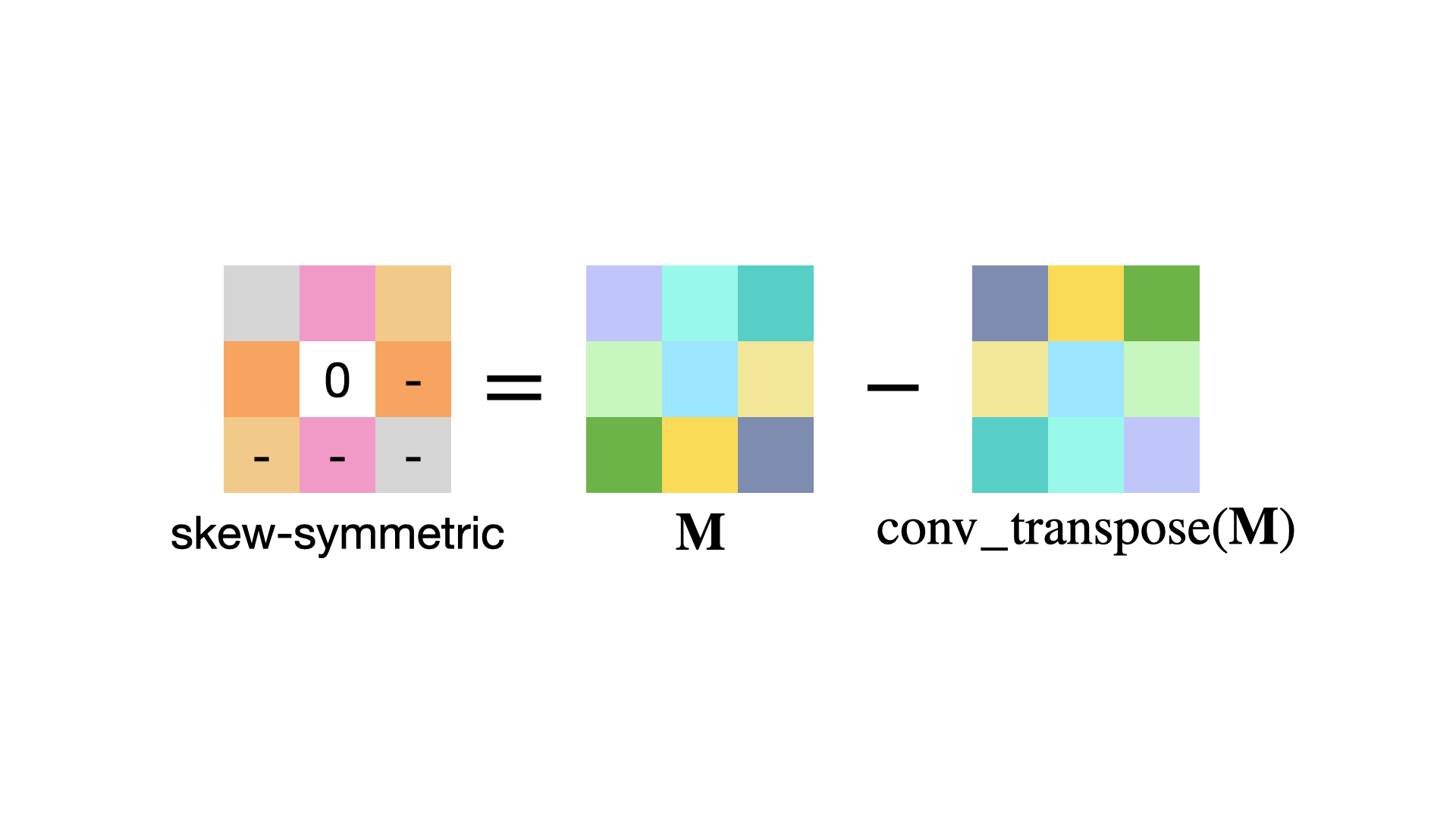}
\caption{Constructing a Skew-Symmetric convolution filter}
\label{subfig:skew2d}
\end{subfigure}\quad
\begin{subfigure}{0.25\linewidth}
\centering
\includegraphics[trim=9cm 7cm 5cm 7cm, clip, width=\linewidth]{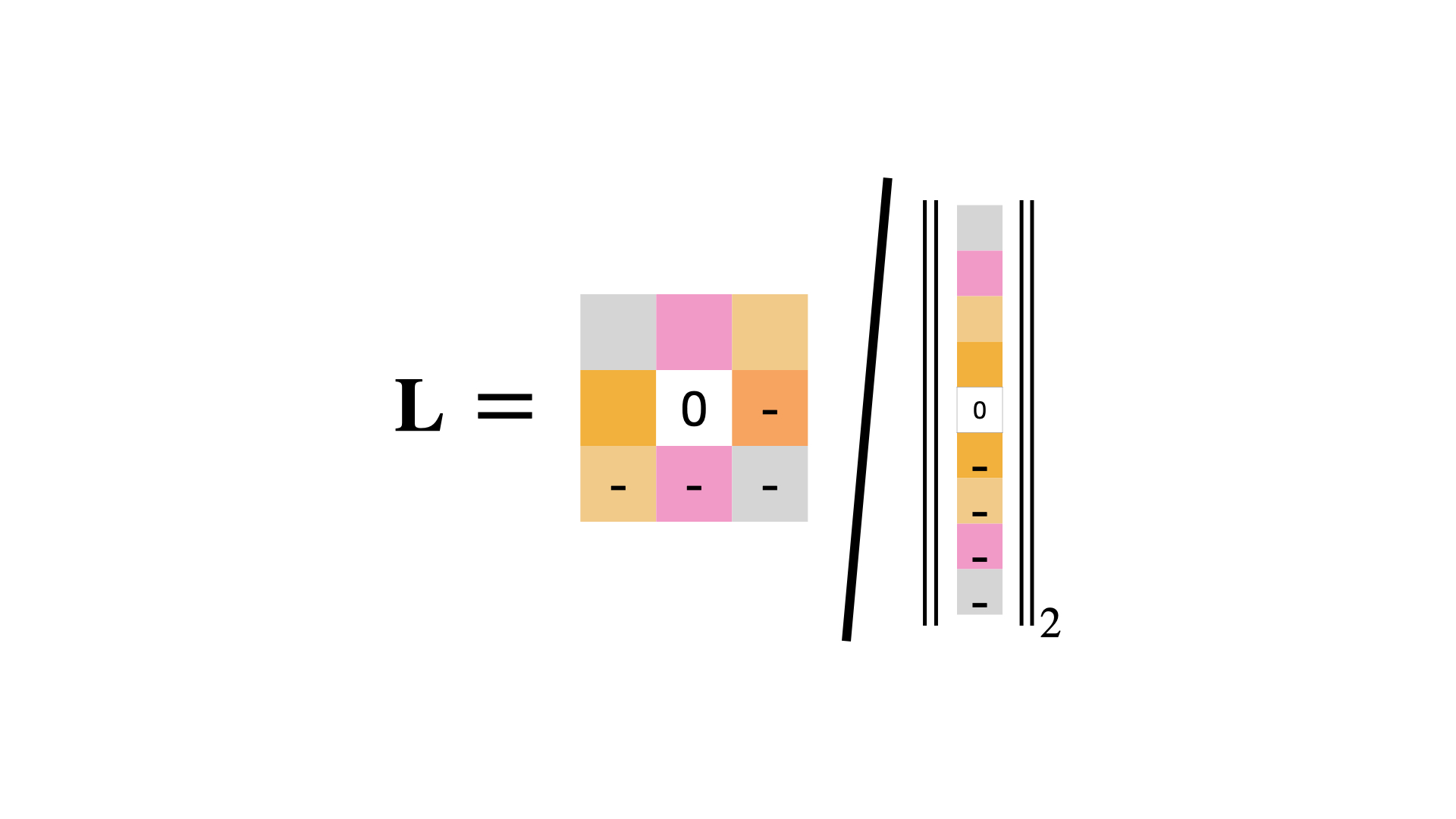}
% \caption{\methodabv\ applied to input $\bX$}
\caption{Spectral normalization}
\label{subfig:spectral_norm}
\end{subfigure}\quad 
\begin{subfigure}{0.4\linewidth}
\centering
\includegraphics[trim=1cm 10cm 1cm 10cm, clip, width=\linewidth]{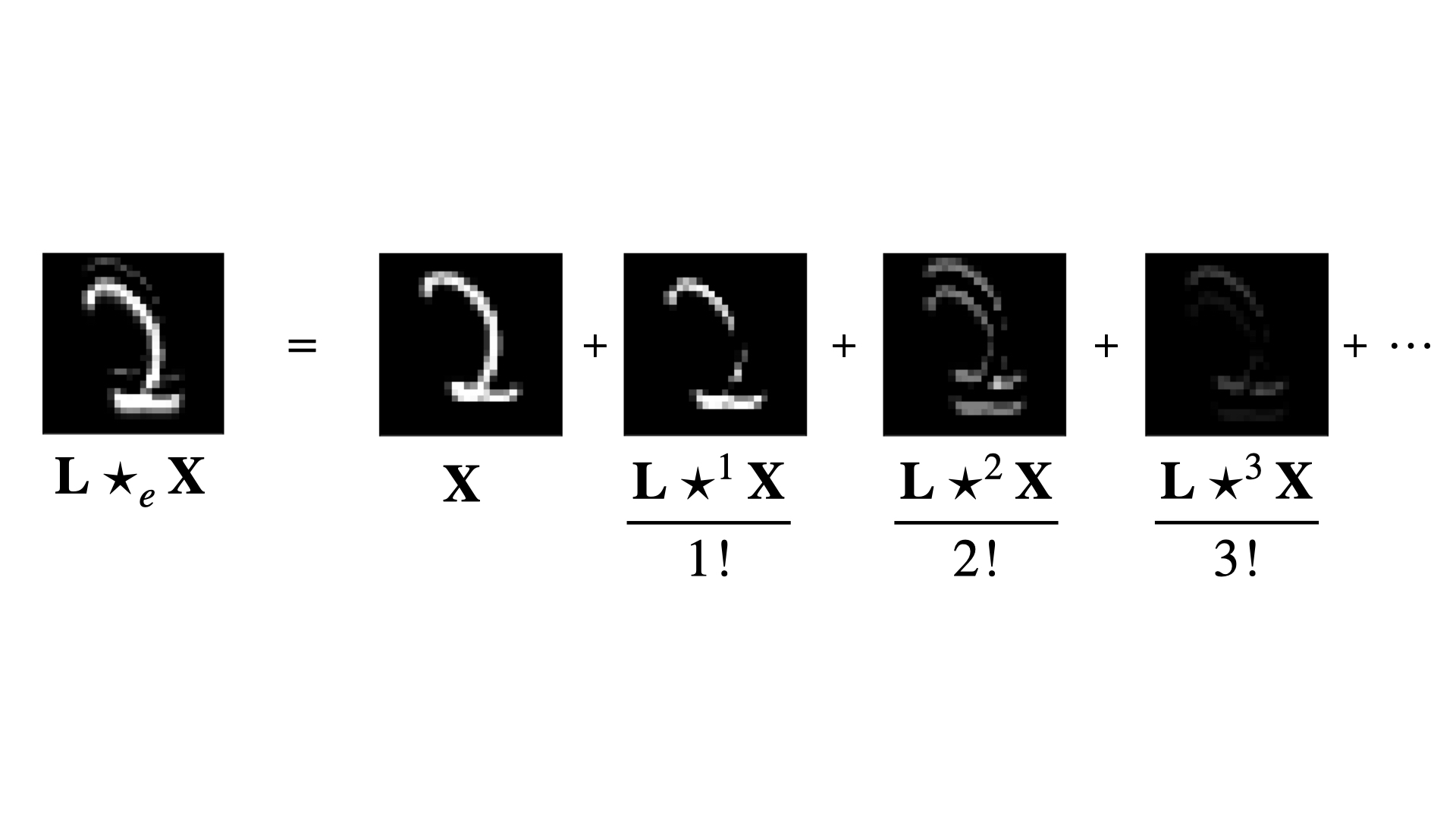}
% \caption{\methodabv\ applied to input $\bX$}
\caption{Convolution exponential $(\bL \star_{e} \bX)$}
\label{subfig:conv2d_trunc}
\end{subfigure}
\caption{Each color denotes a scalar, the minus sign $(-)$ on top of some color denotes the negative of the scalar with that color. Given any convolution filter $\bM$, we can construct a Skew-Symmetric filter (Figure \ref{subfig:skew2d}). Next, we apply spectral normalization to bound the norm of the Jacobian (Figure \ref{subfig:spectral_norm}). On input $\bX$, applying convolution exponential $(\bL \star_{e} \bX)$ results in an Orthogonal convolution (Figure \ref{subfig:conv2d_trunc}). }
%The notation, $\bL \star^{n} \bX = \bL \star^{n-1}\left(\bL \star \bX\right)$}
\label{fig:sop_demo}
\end{figure*}

\citet{li2019lconvnet} introduced an Orthogonal convolution layer called \textbf{B}lock \textbf{C}onvolutional \textbf{O}rthogonal \textbf{P}arametrization (BCOP). BCOP uses a clever application of $\mathrm{1D}$ Orthogonal convolution filters of sizes $2 \times 1$ and $1 \times 2$ to construct a $\mathrm{2D}$ Orthogonal convolution filter. It overcomes common issues of Lipschitz-constrained networks such as gradient norm attenuation and loose lipschitz bounds and enables training of large, provably 1-Lipschitz Convolutional Neural Networks (CNNs) achieving results competitive with existing methods for provable adversarial robustness. However, BCOP suffers from slow training, significant reduction in accuracy and provides no guarantees on its approximation of an Orthogonal Jacobian matrix (details in Section \ref{sec:related_work}).

To address these shortcomings, we introduce an Orthogonal convolution layer called \methodnamebold\ (\methodabv). For provably Lipschitz CNNs, \methodabv\ results in significantly improved standard and certified robust accuracies compared to BCOP while requiring significantly less training time (Table \ref{table:provable_robust_acc}). We also derive provable guarantees on our approximation of an Orthogonal Jacobian. 

Our work is based on the following key mathematical property: If $\bA$ is a Skew-Symmetric matrix (i.e.  $\bA = -\bA^{T}$), $\exp(\bA)$ is an Orthogonal matrix (i.e. $\exp(\bA)^{T}\exp(\bA)=\exp(\bA)\exp(\bA)^{T}=\bI$) where
\begin{align}
 \exp(\bA) =  \bI + \frac{\bA}{1!} + \frac{\bA^{2}}{2!} + \frac{\bA^{3}}{3!} \cdots\ =\ \sum_{i=0}^{\infty} \frac{\bA^{i}}{i!}. \quad \label{eq:intro_exp}
\end{align}
% The above series converges for \emph{all} matrices $\bA$.
To design an Orthogonal convolution layer using this property, we need to: (a) construct \emph{Skew-Symmetric filters}, i.e. convolution filters whose Jacobian is Skew-Symmetric; and (b) efficiently approximate $\exp(\bJ)$ with a guaranteed small error where $\bJ$ is the Jacobian of a Skew-Symmetric filter. 

To construct Skew-Symmetric convolution filters, we prove (in Theorem \ref{thm:main_theorem}) that every Skew-Symmetric filter $\bL$ can be written as $\bL = \bM - \convtransposeoperator(\bM)$ for some filter $\bM$ where $\convtransposeoperator$ represents the {\it convolution transpose operator} defined in equation \eqref{eq:conv2d_h} (note that this operator is different from the matrix transpose). This result is analogous to the property that every real Skew-Symmetric matrix $\bA$ can be written as $\bA = \bB - \bB^{T}$ for some real matrix $\bB$. %Example for a $\mathrm{1D}$ filter is given in Figure \ref{fig:transpose_demo}. 

%To design an Orthogonal convolution using this property, we need to: (a) construct \emph{Skew-Symmetric filters}, i.e. convolution filters whose Jacobian is Skew-Symmetric (example in Figure \ref{fig:transpose_demo}, our result in Theorem \ref{thm:main_theorem}) and (b) efficiently approximate $\exp(\bJ)$ with low error where $\bJ$ is Jacobian of such a filter.

We can efficiently approximate $\exp(\bJ)$ using a finite number of terms in equation \eqref{eq:intro_exp} and the convolution exponential \cite{Hoogeboom2020TheCE}. But it is unclear whether the series can be approximated with high precision and how many terms need to be computed to achieve the desired approximation error. To resolve these issues, we derive a bound on the $l_{2}$ norm of the difference between $\exp(\bJ)$ and its approximation using the first $k$ terms in equation \eqref{eq:intro_exp}, called $\bS_{k}(\bJ)$ when $\bJ$ is Skew-Symmetric (Theorem \ref{thm:approx_error}):
%In Theorem \ref{thm:approx_error}, we prove that for $\bJ$ Skew-Symmetric, we have:
\begin{align}
    \|\exp(\bJ) - \bS_{k}(\bJ)\|_{2} \leq \frac{\|\bJ\|_{2}^{k}}{k!}. \label{eq:intro_approx_error}
\end{align}

%, $\exp(\bJ)$ can be efficiently approximated using finite number of terms in equation \eqref{eq:intro_exp}. 

%whether such an approximation can be efficiently applied to an input $\bX$. what will be the error of such an approximation. Moreover,  because $\bJ$ can be a very large matrix, it is unclear how to efficiently apply the finite term approximation to an input $\bX$. To resolve the second problem, we use the convolution exponential \cite{Hoogeboom2020TheCE}:
%\begin{align*}
%    \sum_{i=0}^{k-1} \frac{\bJ^{i}\bX}{i!} = \sum_{i=0}^{k-1} \frac{\bL \star^{i} \bX}{i!},\qquad \bL \star^{i} \bX = \bL \star^{i-1} (\bL \star \bX)
%\end{align*}
This guarantee suggests that when $\|\bJ\|_{2}$ is small, $\exp(\bJ)$ can be approximated with high precision using a small number of terms. Also, the factorial term in denominator causes the error to decay very fast as $k$ increases. In our experiments, we observe that using $k=12,\ \|\bJ\|_{2} \leq 1.8$ leads to an error bound of $2.415 \times 10^{-6}$. %Since $\bJ$ is the Jacobian of a Skew-Symmetric convolution filter, 
%We can use spectral normalization \cite{miyato2018spectral} to ensure $\|\bJ\|_{2}$ is bounded using the result by \citet{boundsingular}.
We can use spectral normalization \cite{miyato2018spectral} to ensure $\|\bJ\|_{2}$ is provably bounded using the theoretical result of \citet{boundsingular}. The design of \methodabv\ is summarized in Figure \ref{fig:sop_demo}. Code is available at \url{https://github.com/singlasahil14/SOC}.

To summarize, we make the following contributions:
\begin{itemize}
    \item We introduce an Orthogonal convolution layer (called \methodnamebold\ or \methodabv) by first designing a Skew-Symmetric convolution filter (Theorem \ref{thm:main_theorem}) and then computing the exponential function of its Jacobian using a finite number of terms in its Taylor series. 
    
    \item For a Skew-Symmetric filter with Jacobian $\bJ$, we derive a bound on the approximation error between $\exp\left(\bJ\right)$ and its $k$-term approximation (Theorem \ref{thm:approx_error}).

%    that uses the mathematical property that for a Skew-Symmetric matrix $\bA = -\bA^{T}$, $\exp(\bA)$ is an Orthogonal matrix.
    %.\methodabv that implements  by constructing a convolution filter with a Skew-Symmetric jacobian matrix and the applying the convolution exponential using the filter.
%    \item We introduce an Orthogonal convolution layer called \methodnamebold\  (\methodabv)\ that constructs an Orthogonal convolution layer by applying the convolution exponential operation \cite{Hoogeboom2020TheCE} using a Skew-Symmetric convolution filter.
    \item \methodabv\ achieves significantly higher standard and provable robust accuracy on 1-Lipschitz convolutional neural networks than BCOP while requiring less training time (Table  \ref{table:provable_robust_acc}.) For example, \methodabv\ achieves $2.82\%$ higher standard and $3.91\%$ higher provable robust accuracy with $54.6\%$ less training time on CIFAR-10 using the \archname-20 architecture (details in Section \ref{subsec:lipnet_arch}). For deeper networks ($\geq 30$ layers), SOC outperforms BCOP with an improvement of $\geq 10\%$ on both standard and robust accuracy again achieving $\geq 50\%$ reduction in the training time.
    %\item For deeper networks ($\geq 30$ layers), we observe that SOC considerably outperforms BCOP (accuracy difference of $10\%$) with around $50\%$ training time.
    %\item Our method can be generalized to higher order convolution layers (3D convolution etc) and convolution filters with complex weights (Theorems \ref{thm:appendix_3d_kernel} and \ref{thm:appendix_main_theorem_3d}).
    \item In Theorem \ref{thm:approx_cyclic_real}, we prove that for every Skew-Symmetric filter with Jacobian $\bJ$, there exists Skew-Symmetric matrix $\bB$ satisfying: $\exp(\bB)=\exp(\bJ),\  \|\bB\|_{2} \leq \pi$ . Since $\|\bJ\|_{2}$ can be large, this can allow us to reduce the approximation error without sacrificing the expressive power.
    %We derive a similar result for skew-hermitian matrices (Theorem \ref{thm:appendix_approx_cyclic_complex}).
\end{itemize}

\section{Related work}\label{sec:related_work}
\textbf{Provably lipschitz convolutional neural networks}: \citet{Anil2018SortingOL} proposed a class of fully connected neural networks (FCNs) which are Gradient Norm Preserving (GNP) and provably 1-Lipschitz using the GroupSort activation and Orthogonal weight matrices. Since then, there have been numerous attempts to tightly enforce 1-Lipschitz constraints on convolutional neural networks (CNNs) \cite{cisseparseval2017, Tsuzuku2018LipschitzMarginTS, qian2018lnonexpansive, gouk2020regularisation, sedghi2018singular}. However, these approaches either enforce loose lipschitz bounds or are computationally intractable for large networks.  \citet{li2019lconvnet} introduced an Orthogonal convolution layer called \textbf{B}lock \textbf{C}onvolutional \textbf{O}rthogonal \textbf{P}arametrization (BCOP) that avoids the aforementioned issues and allows the training of large, provably 1-Lipschitz CNNs while achieving provable robust accuracy comparable with the existing methods. However, it suffers from some issues: (a) it can only represent a subset of all Orthogonal convolutions, (b) it requires a BCOP convolution filter with $2n$ channels to represent all the connected components of a BCOP convolution filter with $n$ channels thus requiring 4 times more parameters, (c) to construct a convolution filter with size $k \times k$ and $n$ input/output channels, it requires $2k-1$ matrices of size $2n \times 2n$ that must remain Orthogonal throughout training; resulting in well known difficulties of optimization over the Stiefel manifold \cite{Edelman1998TheGO}, (d) it constructs convolution filters from symmetric projectors and error in these projectors can lead to an error in the final convolution filter whereas BCOP does not provide guarantees on the error. 
% (e) it only works with convolution layers that use cyclic padding and periodicity is not a good inductive bias for images and  (f) it is not clear how to extend this approach to construct Orthogonal convolution kernels for higher order convolutions (3D, 4D etc) and for convolution layers where each scalar element in the convolution kernel is a complex number \cite{trabelsi2018deep} (in this case, the Jacobian should be a unitary matrix). We show in Table XXXX that BCOP results in significantly reduced accuracy compared to standard convolution layers for widely used architectures such as Resnet-18, Resnet-34 and Resnet-50. Since the space of BCOP convolutions is disconnected, they propose to use BCOP with $2n$ input/output channels to represent all the connected components of BCOP convolutions with $n$ channels. They also show that all 2D Orthogonal convolutions cannot be represented using BCOP. These networks have singular values that all concentrate near unity, a property known as \textit{dynamical isometry}. This property has been shown to speed up training by orders of magnitude when enforced during weight initialization \cite{Pennington2017ResurrectingTS}. Anil's paper \cite{Anil2018SortingOL}, universal function approximation theorem for Lipschitz functions in the $l_{\infty}$ norm. Li's paper \cite{li2019lconvnet} on Orthogonal convolutions using BCOP.

\textbf{Provable defenses against adversarial examples}: A classifier is said to be provably robust if one can guarantee that a classifier’s prediction remains constant within some region around the input. Most of the existing methods for provable robustness either bound the Lipschitz constant of the neural network or the individual layers \cite{weng2018CertifiedRobustness, zhang2018recurjac, zhang2018crown, Wong2018ScalingPA, Wong2017ProvableDA, Raghunathan2018SemidefiniteRF, croce2019provable, Singh2018FastAE, 2020curvaturebased}. However, these methods do not scale to large and practical networks on ImageNet. To scale to such large networks, randomized smoothing \cite{Liu2018TowardsRN, Cao2017MitigatingEA, Lcuyer2018CertifiedRT, Li2018CertifiedAR, Cohen2019CertifiedAR, Salman2019ProvablyRD, cursedimensionalitykumar20, levine2019certifiably} has been proposed as a \textit{probabilistically certified defense}. %While this can scale to large networks, certifying robustness with high probability requires generating a large number of noisy samples leading to high inference-time computational complexity. 
In contrast, the defense we propose in this work is deterministic and hence not directly comparable to randomized smoothing.

\section{Notation}\label{sec:notation}
%For $n \in \mathbb{N}$, we use $[n]$ to denote the set $\{0, \dots , n\}$. 
For a vector $\bv$, $\bv_{j}$ denotes its $j^{th}$ element. For a matrix $\bA$, $\bA_{j,:}$ and $\bA_{:,k}$ denote the $j^{th}$ row and $k^{th}$ column respectively. Both $\bA_{j,:}$ and $\bA_{:,k}$ are assumed to be column vectors (thus $\bA_{j,:}$ is the transpose of $j^{th}$ row of $\bA$). $\bA_{j,k}$ denotes the element in $j^{th}$ row and $k^{th}$ column of $\bA$. %$\bA_{j,:k}$ and $\bA_{:j,k}$ denote the vectors containing the first $k$ elements of the $j^{th}$ row and first $j$ elements of $k^{th}$ column, respectively. 
$\bA_{:j,:k}$ denotes the matrix containing the first $j$ rows and $k$ columns of $\bA$. The same rules are directly extended to higher order tensors. Bold zero (i.e. $\mathbf{0}$) denotes the matrix (or tensor) consisting of zero at all elements and $\bI$ denotes the identity matrix. $\otimes$ denotes the kronecker product. We use $\mathbb{C}$ to denote the field of complex numbers and $\mathbb{R}$ for real numbers. For a scalar $a \in \mathbb{C}$, $\overline{a}$ denotes its complex conjugate. For a matrix (or tensor) $\bA$, $\overline{\bA}$ denotes the element-wise complex conjugate. For $\bA \in \mathbb{C}^{m \times n}$, $\bA^{H}$ denotes the Hermitian transpose (i.e. $\bA^{H} = \overline{\bA^{T}}$). For $a \in \mathbb{C}$, $\operatorname{Re}(a)$, $\operatorname{Im}(a)$ and $|a|$ denote the real part, imaginary part and modulus of $a$, respectively. We use $\iota$ to denote the imaginary part \textit{iota} (i.e. $\iota^{2}=-1$). 

For a matrix $\bA \in \mathbb{C}^{q \times r}$ and a tensor $\bB \in \mathbb{C}^{p \times q \times r}$, $\overrightarrow{\bA}$ denotes the vector constructed by stacking the rows of $\bA$ and $ \overrightarrow{\bB}$ by stacking the vectors $\overrightarrow{\bB_{j,:,:}},\ j \in [p-1]$ so that: 
\begin{align*}
&\left(\overrightarrow{\bA}\right)^{T} = \begin{bmatrix}
\bA_{0,:}^{T}\ ,\ \bA_{1,:}^{T}\ ,\ \hdots\ ,\ \bA_{q-1,:}^{T}\end{bmatrix} \\
&\left(\overrightarrow{\bB}\right)^{T} = \begin{bmatrix}
\left(\overrightarrow{\bB_{0,:,:}}\right)^{T}\ ,\ \left(\overrightarrow{\bB_{1,:,:}}\right)^{T}\ ,\ \hdots\ ,\ \left(\overrightarrow{\bB_{p-1,:,:}}\right)^{T} \end{bmatrix}
\end{align*}

For a $2$D convolution filter, $\bL \in \mathbb{C}^{p \times q \times r \times s}$, we define the tensor $\convtransposeoperator(\bL) \in \mathbb{C}^{q \times p \times r \times s}$ as follows:
\begin{align}
 [\mathrm{conv\_transpose}(\bL)]_{i,j,k,l} = \overline{[\bL]}_{j,i,r-1-k,s-1-l}  \label{eq:conv2d_h}
\end{align}
Note that this is very different from the usual matrix transpose. See an example in Section \ref{sec:skew_filters}. Given an input $\bX \in \mathbb{C}^{q \times n \times n}$, we use $\bL \star \bX \in \mathbb{C}^{p \times n \times n}$ to denote the convolution of filter $\bL$ with $\bX$. We use the the notation $\bL \star^{i} \bX \triangleq \bL \star^{i-1}\left(\bL \star \bX\right)$. Unless specified, we assume zero padding and stride 1 in each direction.

\begin{figure}[t]
\centering
\includegraphics[trim=0cm 0cm 5cm 15cm, clip,width=0.8\linewidth]{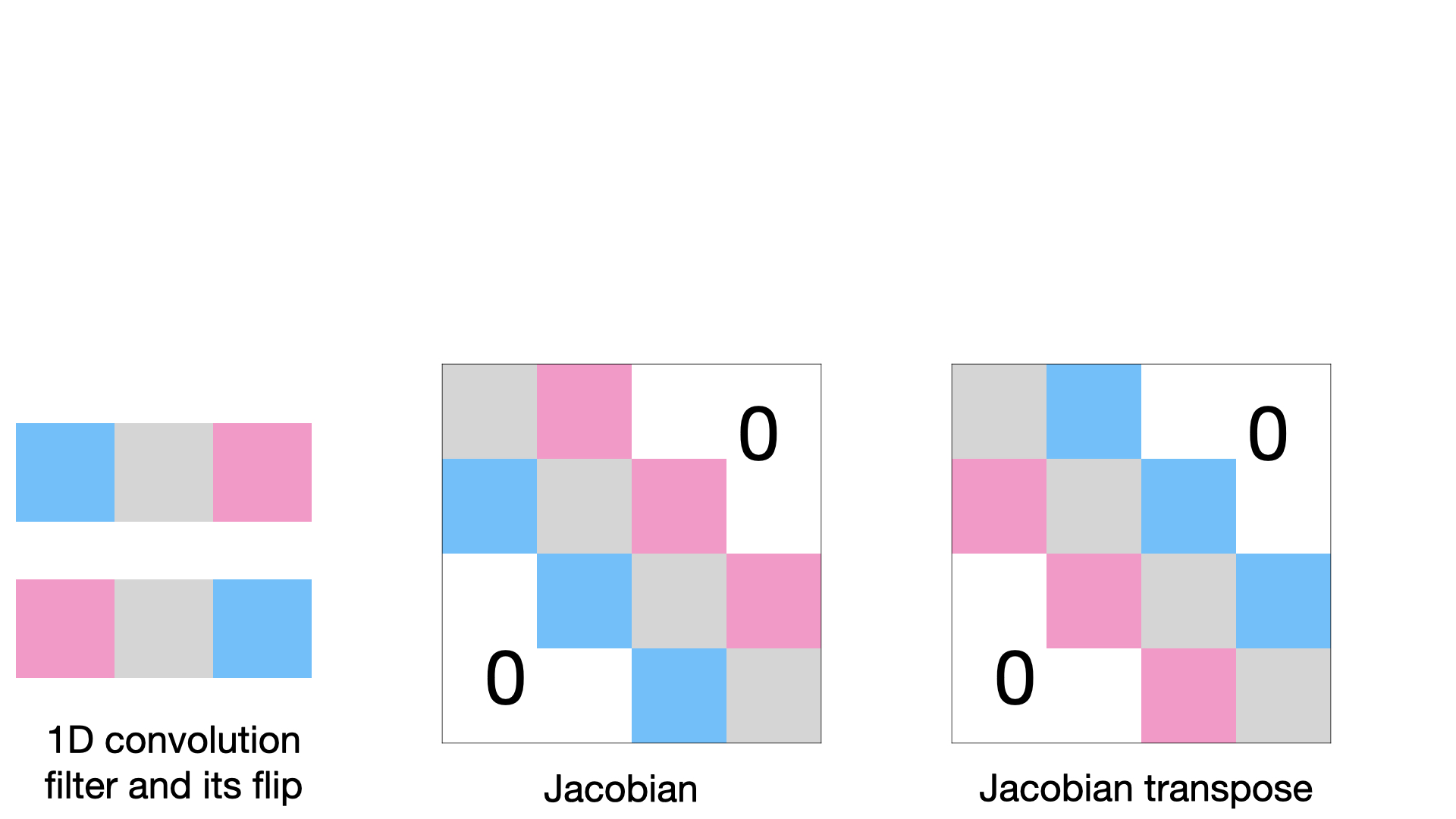}
\caption{Each color denotes a scalar. Flipping a conv. filter (of odd size) transposes its Jacobian. Thus, any odd-sized filter that equals the negative of its flip leads to a Skew-Symmetric Jacobian.}
\label{fig:transpose_demo}
\end{figure}

\section{Filters with Skew Symmetric Jacobians}\label{sec:skew_filters}
We know that for any matrix $\bA$ that is Skew-Symmetric ($\bA = -\bA^{T}$),\ $\exp(\bA)$ is an Orthogonal matrix:
\begin{align*}
    \exp(\bA) \left(\exp(\bA)\right)^{T} = \left(\exp(\bA)\right)^{T}\exp(\bA) = \bI 
\end{align*}
This suggests that if we can parametrize the complete set of convolution filters with Skew-Symmetric Jacobians, we can use the convolution exponential \cite{Hoogeboom2020TheCE} to approximate an Orthogonal matrix. To construct this set, we first prove that, if convolution using filter $\bL \in \mathbb{R}^{m \times m \times p \times q}$ ($p$ and $q$ are odd) has Jacobian $\bJ$, the convolution using $\convtransposeoperator(\bL)$ results in Jacobian $\bJ^{T}$. 

To motivate our proof, consider a filter $\bL \in \mathbb{R}^{1 \times 1 \times 3 \times 3}$. Applying $\convtransposeoperator$ (equation \eqref{eq:conv2d_h}), we get:
\begin{align*}
\bL = \begin{bmatrix}
a & b & c\\
d & e & f\\
g & h & i
\end{bmatrix},\quad
\convtransposeoperator\left(\bL\right) = \begin{bmatrix}
i & h & g\\
f & e & d\\
c & b & a
\end{bmatrix}
\end{align*}
That is, for a $\mathrm{2D}$ convolution filter with 1 channel, $\convtransposeoperator$
flips it along the horizontal and vertical directions. To understand why this flipping transposes the Jacobian, we provide another example for a $\mathrm{1D}$ convolution filter in Figure \ref{fig:transpose_demo}. Our proof uses the following expression for the Jacobian of convolution using a filter $\bL \in \mathbb{R}^{1 \times 1 \times (2p+1) \times (2q+1)}$ and input $\bX \in \mathbb{R}^{1 \times n \times n}$:
$$\bJ = \sum_{i=-p}^{p} \sum_{j=-q}^{q} \bL_{0, 0, p+i, q+j} \left(\bP^{(i)} \otimes \bP^{(j)}\right)$$
where $\bP^{(k)} \in \mathbb{R}^{n \times n}$, $\bP^{(k)}_{i,j}=1$ if $i-j=k$ and $0$ otherwise. The above equation leads to the following theorem:
\begin{theorem}\label{thm:2d_kernel}
	Consider a $\mathrm{2D}$ convolution filter $\bL \in \mathbb{R}^{m \times m \times (2p+1) \times (2q+1)}$ and input $\bX  \in \mathbb{R}^{m \times n \times n}$. Let $\bJ = \nabla_{\overrightarrow{\bX}} \ \overrightarrow{\left(\bL \star \bX\right)}$, then  $\bJ^{T}=\nabla_{\overrightarrow{\bX}} \ \overrightarrow{\left(\convtransposeoperator(\bL) \star \bX\right)}$.
\end{theorem}

Next, we prove that any $\mathrm{2D}$ convolution filter $\bL$ whose Jacobian is a Skew-Symmetric matrix can be expressed as: $\bL = \bM - \convtransposeoperator(\bM)$ where $\bM$ has the same dimensions as $\bL$. This allows us to parametrize the set of all convolution filters with Skew-Symmetric Jacobian matrices. % using the following theorem:
\begin{theorem}\label{thm:main_theorem}
	Consider a $\mathrm{2D}$ convolution filter $\bL \in \mathbb{R}^{m \times m \times (2p+1) \times (2q+1)}$ and input $\bX \in \mathbb{R}^{m \times n \times n}$. The Jacobian $\nabla_{\overrightarrow{\bX}} \ \overrightarrow{\left(\bL \star \bX\right)}$ is Skew-Symmetric if and only if:
	$$\bL = \bM - \convtransposeoperator(\bM) $$
	for some filter $\bM \in \mathbb{R}^{m \times m \times (2p+1) \times (2q+1)}$.
\end{theorem}
We prove Theorems \ref{thm:2d_kernel} and \ref{thm:main_theorem} for the more general case of complex convolution filters ($\bL_{i,j,k,l} \in \mathbb{C}$) in Appendix Sections \ref{subsec:proofs_appendix_2d_kernel} and \ref{subsec:proofs_appendix_main_theorem}. Theorem \ref{thm:main_theorem} allow us to convert any arbitrary convolution filter into a filter with a Skew-Symmetric Jacobian. This leads to the following definition: 
% 	\begin{definition}{(\textbf{Skew-Symmetric convolution filter})}\label{def:skew_symmetric_filter}
% 		A convolution filter $\bL \in \mathbb{R}^{m \times m \times h \times w}$ is said to be a Skew-Symmetric filter if given an input $\bX \in \mathbb{R}^{m \times n \times n}\ (n \geq h, n \geq w)$ and the output $\bY = \bL \star \bX$. The Jacobian of $\bY$ with respect to $\bX$ is a Skew-Symmetric matrix.
% 	\end{definition}
	\begin{definition}{(\textbf{Skew-Symmetric Convolution Filter})}\label{def:skew_symmetric_filter}
		A convolution filter $\bL \in \mathbb{R}^{m \times m \times (2p+1) \times (2q+1)}$ is said to be Skew-Symmetric if given an input $\bX  \in \mathbb{R}^{m \times n \times n}$, the Jacobian matrix  $\nabla_{\overrightarrow{\bX}} \ \overrightarrow{\left(\bL \star \bX\right)}$ is Skew-Symmetric.
	\end{definition}
We note that although Theorem \ref{thm:main_theorem} requires the height and width of $\bM$ to be odd integers, we can also construct a Skew-Symmetric filter when $\bM$ has even height/width by zero padding $\bM$ to make the desired dimensions odd. 
%We derive similar results for 3D convolution filters in Appendix Sections \ref{subsec:proofs_appendix_3d_kernel} and \ref{subsec:proofs_appendix_main_theorem_3d} (Theorems  \ref{thm:appendix_3d_kernel} and \ref{thm:appendix_main_theorem_3d}). 

%Moreover, the above results allow us to construct Skew-Symmetric convolution filters in a straightforward manner. 
%	\begin{definition}{(\textbf{skew-hermitian convolution filter})}\label{def:skew_hermitian_filter}
%		A convolution filter $\bL \in \mathbb{C}^{m \times m \times h \times w}$ is said to be a skew hermitian filter if given an input $\bX \in \mathbb{C}^{m \times n \times n}\ (n \geq h, n \geq w)$ and the output $\bY = \bL \star \bX$. The Jacobian of $\bY$ with respect to $\bX$ is a skew hermitian matrix.\end{definition}

\section{Skew Orthogonal Convolution layers}
%In the previous section, we described our method for constructing a Skew-Symmetric convolution filter. 
In this section, we derive a method to approximate the exponential of the Jacobian of a Skew-Symmetric convolution filter (i.e. $\exp(\bJ)$). We also derive a bound on the approximation error. Given an input $\bX \in \mathbb{R}^{m \times n \times n}$ and a Skew-Symmetric convolution filter $\mathbf{L} \in \mathbb{R}^{m \times m \times k \times k}$ ($k$ is odd), let $\bJ$ be the Jacobian of convolution filter $\bL$ so that: 
\begin{align}
    \bJ\overrightarrow{\bX} = \overrightarrow{\left(\bL \star \bX\right)} \label{eq:conv_equiv}
\end{align}
By construction, we know that $\bJ$ is a Skew-Symmetric matrix, thus $\exp(\bJ)$ is an Orthogonal matrix. We are interested in computing $\exp\left(\bJ\right)\overrightarrow{\bX}$ efficiently where:
\begin{align}
    &\exp\left(\bJ\right)\overrightarrow{\bX} = \overrightarrow{\bX} + \frac{\bJ \overrightarrow{\bX}}{1!} + \frac{\bJ^{2} \overrightarrow{\bX}}{2!} + \frac{\bJ^{3} \overrightarrow{\bX}}{3!} + \cdots \nonumber 
\end{align}
Using equation \eqref{eq:conv_equiv}, the above expression can be written as:
\begin{align}
    &\exp\left(\bJ\right)\overrightarrow{\bX} = \overrightarrow{\bX} + \overrightarrow{\frac{\bL \star \bX}{1!}} + \overrightarrow{\frac{\bL \star^{2} \bX}{2!}} + \overrightarrow{\frac{\bL \star^{3} \bX}{3!}} + \cdots \nonumber
\end{align}
where the notation $\bL \star^{i} \bX \triangleq \bL \star^{i-1} (\bL \star \bX)$. Using the above equation, we define $\bL \star_{e} \bX$ as follows:
\begin{align}
    &\bL \star_{e}\bX=\bX+\frac{\bL \star \bX}{1!} + \frac{\bL \star^{2} \bX}{2!} + \frac{\bL \star^{3} \bX}{3!} + \cdots \label{eq:conv_exp_series} 
\end{align}
The above operation is called \textit{convolution exponential}, and was introduced by \citet{Hoogeboom2020TheCE}. By construction, $\bL \star_{e} \bX$ satisfies: $\exp\left(\bJ\right)\overrightarrow{\bX} = \overrightarrow{\bL \star_{e} \bX}$. Thus, the Jacobian of $\overrightarrow{\bL \star_{e} \bX}$ with respect to $\overrightarrow{\bX}$ is equal to $\exp(\bJ)$ which is Orthogonal (since $\bJ$ is Skew-Symmetric). However, $\bL \star_{e} \bX$ can only be approximated using a finite number of terms in the series given in equation \eqref{eq:conv_exp_series}. Thus, we need to bound the error of such an approximation.
%$(\star_{e})$ of the filter $\bM$ applied to the input $\bX$\ \cite{Hoogeboom2020TheCE}. $\bM \star_{e} \bX$ satisfies:
%\begin{align}
%    &\exp\left(\bJ\right)\overrightarrow{\bX} = \overrightarrow{\bM \star_{e} \bX} \nonumber
%\end{align}
%Both $\exp\left(\bJ\right)$ and $\bM \star_{e} \bX$ can be computed using the series:
% \begin{align}
%     &\exp\left(\bJ\right)\overrightarrow{\bX} = \overrightarrow{\bX} + \frac{\bJ \overrightarrow{\bX}}{1!} + \frac{\bJ^{2} \overrightarrow{\bX}}{2!} + \frac{\bJ^{3} \overrightarrow{\bX}}{3!} + \cdots \nonumber \\
%     &\bM \star_{e} \bX = \bX + \frac{\bM \star \bX}{1!} + \frac{\bM \star (\bM \star \bX)}{2!} + \cdots \label{eq:conv_exp_series} 
% \end{align}

\subsection{Bounding the Approximation Error}\label{subsec:approx_error}
To bound the approximation error using a finite number of terms, first note that since the Jacobian matrix $\bJ$ is Skew-Symmetric, all the eigenvalues are purely imaginary. For a purely imaginary scalar $\lambda \in \mathbb{C}$ \big(i.e. $\operatorname{Re}(\lambda)=0$\big), we first bound the error between $\exp(\lambda)$ and approximation $p_{k}(\lambda)$ computed using $k$ terms of the exponential series as follows:
\begin{align}
    \big| \exp(\lambda) - p_{k}(\lambda) \big| \leq \frac{|\lambda|^{k}}{k!},\qquad \forall\ \lambda: \operatorname{Re}(\lambda)=0 \label{eq:error_bound_scalar}
\end{align}
The above result then allows us to prove the following result for a Skew-Symmetric matrix in a straightforward manner:
\begin{theorem}\label{thm:approx_error}
    %\begin{enumerate}[label=(\alph*)]
    %\item For a scalar $\lambda \in \mathbb{C}$ with $\operatorname{Re}(\lambda)=0$, the error between $\exp(\lambda)$ and approximation $p_{k}(\lambda)$ given below can be bounded using equation \eqref{eq:error_bound_scalar}:
    %\begin{align*}
	% &\exp(\lambda) = \sum_{i=0}^{\infty} \frac{\lambda^{k}}{i!},\qquad p_{k}(\lambda) = \sum_{i=0}^{k-1} \frac{\lambda^{i}}{i!} \nonumber
    %\end{align*}
    For Skew-Symmetric $\bJ$, we have the inequality: %the following inequality:
    %the error between $\exp(\bJ)$ and the k-term approximation can be bounded as follows:%$\bS_{k}(\bA)$ 
    \begin{align*}
	 &\|\exp(\bJ) - \bS_{k}(\bJ)\|_{2} \leq \frac{\|\bJ\|_{2}^{k}}{k!}\ \ \  \text{where}\ \bS_{k}(\bJ) = \sum_{i=0}^{k-1} \frac{\bJ^{i}}{i!}\quad \nonumber 
%	 &\exp(\bJ) = \sum_{i=0}^{\infty} \frac{\bJ^{i}}{i!},\qquad \bS_{k}(\bJ) = \sum_{i=0}^{k-1} \frac{\bJ^{i}}{i!} \nonumber
	\end{align*}
    %\end{enumerate}
\end{theorem}
A more general proof of Theorem \ref{thm:approx_error} (for $\bJ \in \mathbb{C}^{n \times n}$ and skew-Hermitian i.e.  $\bJ = -\bJ^{H}$) is given in Appendix Section \ref{subsec:proofs_appendix_approx_error}. The above theorem allows us to bound the approximation error between the true matrix exponential (which is Orthogonal) and its $k$ term approximation as a function of the number of terms ($k$) and the Jacobian norm $\|\bJ\|_{2}$. The factorial term in the denominator causes the error to decay very fast as the number of terms increases. We call the resulting algorithm \methodnamebold\ (\methodabv).

We emphasize that the above theorem is valid only for Skew-Symmetric matrices and hence not directly applicable for the convolution exponential \cite{Hoogeboom2020TheCE}. %Moreover, since BCOP constructs an Orthogonal convolution filter from symmetric projector matrices, any error in symmetric projectors can lead to an error in the final convolution filter. But BCOP provides no such guarantee on its approximation. 
%We emphasize that since BCOP constructs an Orthogonal convolution from symmetric projector matrices, any error in symmetric projector can lead to an error in the convolution kernel. But BCOP provides no such guarantee on its approximation.
%When the Jacobian matrix is derived from a skew-hermitian convolution filter (using definition \ref{def:skew_hermitian_filter}), the parametrization is called \methodnamecomplexbold\ (\methodabvcomplex). 
%When the Jacobian matrix is derived from a Skew-Symmetric convolution filter (using definition \ref{def:skew_symmetric_filter}), it is called \methodnamebold\ (\methodabv). All experiments in this work were conducted using \methodabv.

\begin{algorithm}[t]
\SetAlgoLined
\SetKwInput{KwInput}{Input}                % Set the Input
\SetKwInput{KwOutput}{Output}              % set the Output
\SetKwInput{KwReturn}{Return}              % set the Output
\DontPrintSemicolon
\KwInput{feature map: $\bX \in \mathbb{R}^{c_{i} \times n \times n}$,\ convolution filter: $\bM  \in \mathbb{R}^{m \times m \times h \times w}$ $\left(m = \max(c_{i},\ c_{o})\right)$ ,\ terms: $K$}
\KwOutput{output after applying convolution exponential: $\bY$}
\If {$c_{i} < c_{o}$} 
{
    $\bX' \gets \mathrm{pad}(\bX,\ (c_{o}-c_{i}, 0, 0))$
}
{
$\bL \gets \bM - \convtransposeoperator(\bM)$ \\
$\bL \gets \mathrm{spectral\_normalization}(\bL)$ \\
$\bY \gets \bX'$ \\
$\mathrm{factorial} \gets 1$ \\
}
\For{$j \gets 2$ to $K$}
{
    {
        $\bX' \gets \bL \star \bX'$\\ 
        $\mathrm{factorial} \gets \mathrm{factorial} * (j-1)$ \\
        $\bY \to \bY + \left(\bX'/\mathrm{factorial}\right)$ 
    }
}
\If {$c_{i} > c_{o}$} 
{
    $\bY \gets \bY[0:c_{o},\ :\ ,\ :\ ] $
}
\KwReturn{$\bY$}
\caption{\methodname}
\label{alg:shop}
\end{algorithm}

\subsection{Complete Set of Skew Orthogonal Convolutions}
Observe that for $\operatorname{Re}(\lambda)=0$, (i.e. $\lambda = \iota \theta,\ \theta \in \mathbb{R}$), we have: 
$$\ \ \ \exp(\lambda) = \exp(\lambda + 2 \iota \pi k) = \cos(\theta) + \iota \sin(\theta),\ \ k \in \mathbb{Z}$$
This suggests that we can shift $\lambda$ by integer multiples of $2\pi \iota$ without changing $\exp(\lambda)$ while reducing the approximation error (using Theorem \ref{thm:approx_error}). For example, $\exp(\iota \pi/3)$ requires fewer terms to achieve the desired approximation (using equation \eqref{eq:error_bound_scalar}) than say $\exp(\iota (\pi/3 + 2\pi))$ because the latter has higher norm (i.e. $2\pi + \pi/3 = 7\pi/3$) than the former (i.e. $\pi/3$). This insight leads to the following theorem:
\begin{theorem}\label{thm:approx_cyclic_real}
    Given a real Skew-Symmetric matrix $\bA$, we can construct another real Skew-Symmetric matrix $\bB$ such that $\bB$ satisfies: (i) $\exp(\bA)=\exp(\bB)$ and (ii) $\|\bB\|_{2} \leq \pi$.
\end{theorem}
% \begin{theorem}\label{thm:approx_cyclic_complex}
%     Given a skew-hermitian matrix $\bA$, we can construct a skew-hermitian matrix $\bB$ by adding integer multiples of $2\pi \iota$ to eigenvalues of $\bA$ such that $\bB$ satisfies: (i) $\exp(\bA)=\exp(\bB)$ and (ii) $\|\bB\|_{2} \leq \pi$.
% \end{theorem}
A proof is given in Appendix Section \ref{subsec:proofs_appendix_approx_cyclic_real}. This proves that every real Skew-Symmetric Jacobian matrix $\bJ$ (associated with some Skew-Symmetric convolution filter $\bL$) can be replaced with a Skew-Symmetric Jacobian $\bB$ such that $\exp\left(\bB\right) = \exp\left(\bJ\right)$ and $\|\bB\|_{2} \leq \pi$ \big(note that $\|\bJ\|_{2}$ can be arbitrarily large\big). This strictly reduces the approximation error (Theorem \ref{thm:approx_error}) without sacrificing the expressive power.

We make the following observations about Theorem \ref{thm:approx_cyclic_real}: (a) If $\bJ$ is equal to the Jacobian of some Skew-Symmetric convolution filter, $\bB$ may not satisfy this property, i.e. it may not exhibit the block doubly toeplitz structure of the Jacobian of a $\mathrm{2D}$ convolution filter \cite{sedghi2018singular} and thus may not equal the jacobian of some Skew-Symmetric convolution filter; (b) even if $\bB$ satisfies this property, the filter size of the Skew-Symmetric filter whose Jacobian equals $\bB$ can be very different from that of the filter with Jacobian $\bJ$.

In this sense, Theorem \ref{thm:approx_cyclic_real} cannot directly be used to parametrize the complete set of \methodabv\  because it is not clear how to efficiently parametrize the set of all matrices $\bB$ that satisfy (a) $\|\bB\|_{2} \leq \pi$ and (b) $\exp(\bB) = \exp(\bJ)$ where $\bJ$ is the Jacobian of some Skew-Symmetric convolution filter. We leave this question of efficient parametrization of \methodname\ layers open for future research.
%Nevertheless, we use $\pi$ as a heuristic upper bound on the norm of Jacobian for our experiments and leave the question of efficient parametrization open for future research. 
%We derive a similar Theorem for \methodabvcomplex\ (Theorem \ref{thm:appendix_approx_cyclic_complex}) in Appendix Section \ref{subsec:proofs_appendix_approx_cyclic_complex}.
\subsection{Extensions to $\mathrm{3D}$ and Complex Convolutions}\label{subsec:extensions}
When the matrix $\bA \in \mathbb{C}^{n \times n}$ is skew-Hermitian ($\bA = -\bA^{H}$), then $\exp(\bA)$ is a unitary matrix:
\begin{align*}
\exp(\bA) \left(\exp(\bA)\right)^{H} = \left(\exp(\bA)\right)^{H}\exp(\bA) = \bI
\end{align*}
To use the above property to construct a unitary convolution layer with complex weights, we first define:  %\textit{skew-hermitian convolution filter}:
% 	\begin{definition}{(\textbf{Skew-Symmetric convolution filter})}\label{def:skew_symmetric_filter}
% 		A convolution filter $\bL \in \mathbb{R}^{m \times m \times (2p+1) \times (2q+1)}$ is said to be Skew-Symmetric if given an input $\bX  \in \mathbb{R}^{m \times n \times n}$, the Jacobian matrix  $\nabla_{\overrightarrow{\bX}} \ \overrightarrow{\left(\bL \star \bX\right)}$ is Skew-Symmetric.
% 	\end{definition}

\begin{definition}{(\textbf{Skew-Hermitian Convolution Filter})}\label{def:skew_hermitian_filter}
A convolution filter $\bL \in \mathbb{C}^{m \times m \times (2p+1) \times (2q+1)}$ is said to be Skew-Hermitian if given an input $\bX \in \mathbb{C}^{m \times n \times n}$, the Jacobian matrix  $\nabla_{\overrightarrow{\bX}} \ \overrightarrow{\left(\bL \star \bX\right)}$ is Skew-Hermitian.
%and the output $\bY = \bL \star \bX$, the Jacobian of $\bY$ with respect to $\bX$ is skew-Hermitian.
\end{definition}
Using the extensions of Theorems \ref{thm:2d_kernel} and \ref{thm:main_theorem} for complex convolution filters (proofs in Appendix Sections \ref{subsec:proofs_appendix_2d_kernel} and \ref{subsec:proofs_appendix_main_theorem}), we can construct a $\mathrm{2D}$ Skew-Hermitian convolution filter. Next, using an extension of Theorem \ref{thm:approx_error} for complex Skew-Hermitian matrices (proof in Appendix Section \ref{subsec:proofs_appendix_approx_error}), we can get exactly the same bound on the approximation error. The resulting algorithm is called \methodnamecomplexbold\  (\methodabvcomplex). We also prove an extension of Theorem \ref{thm:approx_cyclic_real} for complex Skew-Hermitian matrices in Appendix Section \ref{subsec:proofs_appendix_approx_cyclic_complex}. We discuss the construction of $\mathrm{3D}$ Skew-Hermitian convolution filters in Appendix Sections \ref{subsec:proofs_appendix_3d_kernel} and \ref{subsec:proofs_appendix_main_theorem_3d}.
\begin{figure}[t]
\centering
\includegraphics[trim=0cm 10cm 0cm 3cm, clip,width=0.6\linewidth]{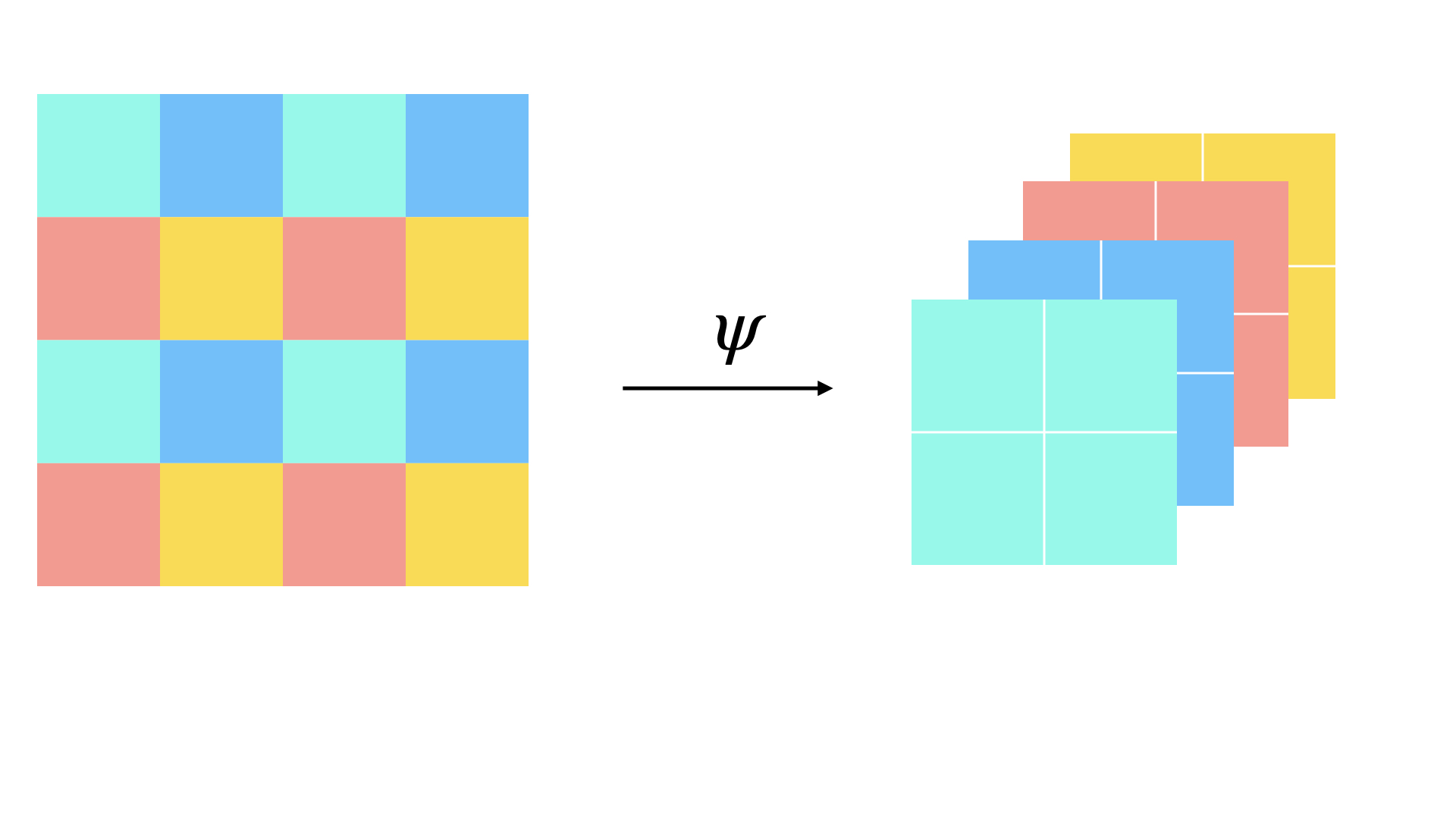}
\caption{Invertible downsampling operation $\psi$}
\label{fig:downsampling}
\end{figure}

\section{Implementation details of \methodabv}
In this section, we explain the key implementation details of \methodabv\ (summarized in Algorithm \ref{alg:shop}). 
% \subsection{Constructing a skew hermitian filter}
% Given an arbitrary $2$D convolution filter $\bM \in \mathbb{R}^{c_{o} \times c_{i} \times h \times w}$, we can easily construct a skew hermitian convolution filter $\bL \in \mathbb{R}^{c_{o} \times c_{i} \times h \times w}$ using Theorem \ref{thm:main_theorem}:
% \begin{align*}
%     \bL = \bM - \mathrm{conv2d\_h}(\bM)
% \end{align*}

\subsection{Bounding the norm of Jacobian}\label{subsec:bound_jacobian}
To bound the norm of the Jacobian of Skew-Symmetric convolution filter, we use the following result:
\begin{thm}{\cite{boundsingular}}\label{thm:sf_singular_bound}
	Consider a convolution filter $\bL \in \mathbb{R}^{c_{o} \times c_{i} \times h \times w}$ applied to input $\bX$. Let $\bJ$ be the Jacobian of $\bL \star \bX$ w.r.t $\bX$, we have the following inequality:
	$$\|\bJ\|_{2} \leq  \sqrt{hw}\ \min\left(\|\bR\|_{2}, \|\bS\|_{2}, \|\bT\|_{2}, \|\bU\|_{2}\right),$$
	where $\bR \in \mathbb{R}^{c_{o}h \times c_{i}w}, \bS \in \mathbb{R}^{c_{o}w \times c_{i}h}, \bT \in \mathbb{R}^{c_{o} \times c_{i}hw}$ and $\bU \in \mathbb{R}^{c_{o}hw \times c_{i}}$ are obtained by reshaping the filter $\bL$.
\end{thm}
Using the above theorem, we divide the Skew-Symmetric convolution filter by $\min\left(\|\bR\|_{2}, \|\bS\|_{2}, \|\bT\|_{2}, \|\bU\|_{2}\right)$ so that the spectral norm of the resulting filter is bounded by $\sqrt{hw}$. We next multiply the normalized filter with the hyperparameter, $0.7$ as we find that it allows faster convergence with no loss in performance. Unless specified, we use $h=w=3$ in all of our experiments resulting in the norm bound of $2.1$.  Note that while the above theorem also allows us to bound the Lipschitz constant of a convolution layer, for deep networks (say $40$ layers), the Lipschitz bound (assuming a $1$-Lipschitz activation function) would increase to $2.1^{40}=7.74 \times 10^{12}$. Thus, the above bound alone is unlikely to enforce a tight global Lipschitz constraint.

%We update all of the norm for all four matrices i.e.$(\|\bR\|_{2}, \|\bS\|_{2}, \|\bT\|_{2}, \|\bU\|_{2})$ using $1$ iteration of power method after each weight update step. We also update the norm after every $200$ weight update steps using $20$ iterations of power method. \\
%We note that Theorem \ref{thm:sf_singular_bound} is applicable only for convolution layers using cyclic padding, while we use zero padding in our experiments. However, we observe that the given normalization works well and the norm of the Jacobian matrix remains bounded below the desired value even with zero padding.

\subsection{Different input and output channels}\label{subsec:diff_in_out}
In general, we may want to construct an orthogonal convolution that maps from $c_{i}$ input channels to $c_{o}$ output channels where $c_{i} \neq c_{o}$. Consider the two cases:

\textbf{Case 1} ($c_{o} < c_{i}$): We construct a Skew-Symmetric convolution filter with $c_{i}$ channels. After applying the exponential, we select the first $c_{o}$ output channels from the output layer.

\textbf{Case 2} ($c_{o} > c_{i}$): We use a Skew-Symmetric convolution filter with $c_{o}$ channels. We zero pad the input with $c_{o} - c_{i}$ channels and then compute the convolution exponential. 

\subsection{Strided convolution}
Given an input $\bX \in \mathbb{R}^{c_{i} \times n \times n}$ ($n$ is even), we may want to construct an orthogonal convolution with output $\bY \in \mathbb{R}^{c_{o} \times (n/2) \times (n/2)}$ (i.e. an orthogonal convolution with stride $2$). To perform a strided convolution, we first apply \textit{invertible downsampling} $\psi$ as shown in Figure \ref{fig:downsampling} \cite{jacobsen2018irevnet} to construct $\bX' \in \mathbb{R}^{4c_{i} \times (n/2) \times (n/2)}$. Next, we apply convolution exponential to $\bX'$ using a Skew-Symmetric convolution filter with $4c_{i}$ input and $c_{o}$ output channels. 

\subsection{Number of terms for the approximation}
During training, we use $6$ terms to approximate the exponential function for speed. During evaluation, we use $12$ terms to ensure that the exponential of the Jacobian is sufficiently close to being an orthogonal matrix. 
%We emphasize that the number of terms for evaluation can be further reduced (per convolution layer) depending on the norm of the Jacobian of the layer after training to achieve the desired error.
\begin{table}[t]
\centering
\renewcommand{\arraystretch}{1.2} 
\begin{tabular}{ p{1.2cm}  p{3.4cm}  p{1.3cm} }
\hline
\multirow{2}{1cm}{\textbf{Output Size}} & \multirow{2}{3cm}{\textbf{Convolution layer}} & \multirow{2}{2cm}{\textbf{Repeats}} \\ 
  &  & \\
\hline
\multirow{2}{*}{$16 \times 16$} & conv $[3 \times 3, 32, 1]$ & $(n/5)-1$\\
%\cline{2-3}
& conv $[3 \times 3, 64, 2]$ & 1 \\
\hline
\multirow{2}{*}{$8 \times 8$} & conv $[3 \times 3, 64, 1]$ & $(n/5)-1$\\
%\cline{2-3}
& conv $[3 \times 3, 128, 2]$ & 1 \\
\hline
\multirow{2}{*}{$4 \times 4$} & conv $[3 \times 3, 128, 1]$ & $(n/5)-1$\\
%\cline{2-3}
& conv $[3 \times 3, 256, 2]$ & 1 \\
\hline
\multirow{2}{*}{$2 \times 2$} & conv $[3 \times 3, 256, 1]$ & $(n/5)-1$\\
%\cline{2-3}
& conv $[3 \times 3, 512, 2]$ & 1 \\
\hline
\multirow{2}{*}{$1 \times 1$} & conv $[3 \times 3, 512, 1]$ & $(n/5)-1$\\
%\cline{2-3}
& conv $[1 \times 1, 1024, 2]$ & 1 \\
\hline
%1 \times 1 & conv $[1 \times 1, K, 1]$ & 1 \\
%\hline
\end{tabular}
\caption{\archname-n Architecture. Each convolution layer is followed by the $\mathrm{MaxMin}$ activation.}
\label{table:architecture}
\end{table}
\subsection{Network architecture}\label{subsec:lipnet_arch}
We design a provably 1-Lipschitz architecture called \archname-$n$ ($n$ is the number of convolution layers and a multiple of $5$ in our experiments). It consists of $(n/5)-1$ Orthogonal convolutions of stride 1 (followed by the $\mathrm{MaxMin}$ activation function), followed by Orthogonal convolution of stride 2 (again followed by the $\mathrm{MaxMin}$). It is summarized in Table \ref{table:architecture}. conv $[k \times k, m, s]$ denotes convolution layer with filter of size $k \times k$, out channels $m$ and stride $s$. It is followed by a fully connected layer to output the class logits. The $\mathrm{MaxMin}$ activation function \cite{Anil2018SortingOL} is described in Appendix Section \ref{sec:appendix_maxmin}. 

%This whole block is repeated 5 times. All convolutions have kernel size $3 \times 3$ except the last layer. The last layer has kernel size $1 \times 1$ and is not followed by the MaxMin activation function. 

\begin{table*}[t]
\centering
\renewcommand{\arraystretch}{1.2} 
\begin{tabular}{ p{2.5cm}  p{1cm} | p{1.5cm}  p{1.5cm}  p{1.5cm} | p{1.5cm}  p{1.5cm}  p{1.5cm} }
\hline
\multirow{3}{*}{\textbf{Model}} & \multirow{3}{1cm}{\textbf{Conv. Type}} & \multicolumn{3}{c|}{\textbf{CIFAR-10}} &  \multicolumn{3}{c}{\textbf{CIFAR-100}} \\
% \cline{3-8}
  &  & \textbf{Standard Accuracy} & \textbf{Robust Accuracy} & \textbf{Time per epoch (s)} & \textbf{Standard Accuracy} & \textbf{Robust Accuracy} & \textbf{Time per epoch (s)}\\
\hline
\multirow{2}{*}{\archname-5} & BCOP & 74.35\% & 58.01\% & 96.153 & 42.61\% & \textbf{28.67\%} & 94.463 \\
%\cline{2-8}
& \textbf{\methodabv} & \textbf{75.78\%} & \textbf{59.16\%} & \textbf{31.096} & \textbf{42.73\%} & 27.82\% & \textbf{30.844} \\
\hline
\multirow{2}{*}{\archname-10} & BCOP & 74.47\% & 58.48\% & 122.115 & 42.08\% & 27.75\% & 119.038 \\
%\cline{2-8}
& \textbf{\methodabv} & \textbf{76.48\%} & \textbf{60.82\%} & \textbf{48.242} & \textbf{43.71\%} & \textbf{29.39\%} & \textbf{48.363}\\
\hline
\multirow{2}{*}{\archname-15} & BCOP & 73.86\% & 57.39\% & 145.944 & 39.98\% & 26.17\% & 144.173 \\
%\cline{2-8}
& \textbf{\methodabv} & \textbf{76.68\%} & \textbf{61.30\%} & \textbf{63.742} & \textbf{42.93\%} & \textbf{28.79\%} & \textbf{63.540} \\
\hline
\multirow{2}{*}{\archname-20} & BCOP & 69.84\% & 52.10\% & 170.009 & 36.13\% & 22.50\% & 172.266 \\
%\cline{2-8}
& \textbf{\methodabv} & \textbf{76.43\%} & \textbf{61.92\%} & \textbf{77.226} & \textbf{43.07\%} & \textbf{29.18\%} & \textbf{76.460} \\
\hline
\multirow{2}{*}{\archname-25} & BCOP & 68.26\% & 49.92\% & 207.359 & 28.41\% & 16.34\% & 205.313 \\
%\cline{2-8}
& \textbf{\methodabv} & \textbf{75.19\%} & \textbf{60.18\%} & \textbf{98.534} & \textbf{43.31\%} & \textbf{28.59\%} & \textbf{95.950} \\
\hline
\multirow{2}{*}{\archname-30} & BCOP & 64.11\% & 43.39\% & 227.916 & 26.87\% & 14.03\% & 229.840 \\
%\cline{2-8}
& \textbf{\methodabv} & \textbf{74.47\%} & \textbf{59.04\%} & \textbf{110.531} & \textbf{42.90\%} & \textbf{28.74\%} & \textbf{107.163} \\
\hline
\multirow{2}{*}{\archname-35} & BCOP & 63.05\% & 41.72\% & 267.272 & 21.71\% & 10.33\% & 274.256 \\
%\cline{2-8}
& \textbf{\methodabv} & \textbf{73.70\%} & \textbf{58.44\%} & \textbf{130.671} & \textbf{42.44\%} & \textbf{28.31\%} & \textbf{126.368} \\
\hline
\multirow{2}{*}{\archname-40} & BCOP & 60.17\% & 38.87\% & 295.350 & 19.97\% & 8.66\% & 289.369 \\
%\cline{2-8}
& \textbf{\methodabv} & \textbf{71.63\%} & \textbf{54.36\%} & \textbf{144.556} & \textbf{41.83\%} & \textbf{27.98\%} & \textbf{140.458}\\
\hline
\end{tabular}
\caption{Results for provable robustness against adversarial examples ($l_{2}$ perturbation radius of 36/255). Time per epoch is the training time per epoch (in seconds).}
\label{table:provable_robust_acc}
\end{table*}

\section{Experiments}
Our goal is to evaluate the expressiveness of our method (\methodabv) compared to BCOP for constructing Orthogonal convolutional layers. To study this, we perform experiments in three settings: (a) provably robust image classification, (b) standard training and (c) adversarial training. %In all these settings, we observe that \methodabv\ significantly outperforms BCOP while requiring significantly less time to train across different network architectures and datasets. 

All experiments were performed using $1$ NVIDIA GeForce RTX 2080 Ti GPU. All networks were trained for 200 epochs with an initial learning rate 0.1, dropped by a factor of 0.1 after 50 and 150 epochs. We use no weight decay for training with BCOP convolution as it significantly reduces its performance. For training with standard convolution and \methodabv, we use a weight decay of $10^{-4}$.

To evaluate the approximation error for \methodabv\ at convergence (using Theorem \ref{thm:approx_error}), we compute the norm of the Jacobian of the Skew-Symmetric convolution filter using real normalization \cite{realsn}. We observe that the maximum norm  (across different experiments and layers of the network) is below $1.8$ (i.e. slightly below the theoretical upper bound of $2.1$ discussed in Section \ref{subsec:bound_jacobian}) resulting in a maximum error of $1.8^{12}/12! = 2.415 \times 10^{-6}$.

\subsection{Provable Defenses against Adversarial Attacks}\label{subsec:provable}
To certify provable robustness of $1$-Lipschitz network $f$ for some input $\bx$, we first define the margin of prediction: $\mathcal{M}_f(\bx) = \max\left(0, y_{t}-\max_{i \neq t} y_{i} \right)$ where $\by = [y_{1},y_{2,\dotsc} ]$ is the predicted logits from $f$ on $\bx$ and $y_{t}$ is the correct logit. Using Theorem 7 in \citet{li2019lconvnet}, we can derive the robustness certificate as $\mathcal{M}_f(\bx)/\sqrt{2}$. The provable robust accuracy, evaluated using an $l_{2}$ perturbation radius of $36/255$ (same as in \citet{li2019lconvnet}) equals the fraction of data points ($\bx$) in the test dataset satisfying $\mathcal{M}_f(\bx)/\sqrt{2} \geq 36/255$. 
% \begin{align*}
%     \mathcal{M}_f(\bx) = \max\left(0, y_{t}-\max_{i \neq t} y_{i} \right)
% \end{align*}

In Table \ref{table:provable_robust_acc}, we show the results of our experiments using different \archname\ architectures with varying number of layers on CIFAR-10 and CIFAR-100 datasets. We make the following observations: (a) \methodabv\  achieves significantly higher standard and provable robust accuracy than BCOP for different architectures and datasets, (b) \methodabv\ requires significantly less training time per epoch than BCOP and (c) as the number of layers increases, the performance of BCOP degrades rapidly but that of \methodabv\ remains largely consistent. For example, on a \archname-40 architecture, \methodabv\ achieves $11.46\%$ higher standard accuracy; $15.49\%$ higher provable robust accuracy on the CIFAR-10 dataset and $21.86\%$ higher standard accuracy; $19.32\%$ higher provable robust accuracy on the CIFAR-100 dataset. We further emphasize that none of the other well known deterministic provable defenses (discussed in Section \ref{sec:related_work}) are scalable to large networks as the ones in Table \ref{table:provable_robust_acc}. BCOP, while scalable, achieves significantly lower standard and provable robust accuracies for deep networks than \methodabv. 

\begin{table*}[!ht]
\centering
\renewcommand{\arraystretch}{1.2} 
\begin{tabular}{  p{1.5cm}  p{2cm} | p{1.5cm}  p{1.5cm} | p{1.5cm}  p{1.5cm}  }
\hline
\multirow{3}{*}{\textbf{Model}} & \multirow{3}{2cm}{\textbf{Conv. Type}} & \multicolumn{2}{c|}{\textbf{CIFAR-10}} &  \multicolumn{2}{c}{\textbf{CIFAR-100}} \\
% \cline{3-6}
  &  & \textbf{Standard Accuracy} & \textbf{Time per epoch (s)} & \textbf{Standard Accuracy} & \textbf{Time per epoch (s)}\\
\hline
\multirow{3}{*}{Resnet-18} & Standard & 95.10\% & 13.289 & 77.60\% & 13.440 \\
\cline{2-6}
%\cmidrule[1pt]{2-6}
% \Cline{12pt}{2-6}
%\specialrule{1pt}{10pt}{0pt}
%\hhline{|~|-|-|-|-|~|}
%\noalign{\vskip\doublerulesep\vskip-\arrayrulewidth}
%\cline{2-6}
%\vskip{-2pt}\\
%\cline{2-6}
%\cline{2-6}
& BCOP & 92.38\% & 128.383 & 71.16\% & 128.146 \\
%\cline{2-6}
& \textbf{\methodabv} & \textbf{94.24\%} & \textbf{110.750} & \textbf{74.55\%} & \textbf{103.633} \\
\hline
\multirow{3}{*}{Resnet-34} & Standard & 95.54\% & 22.348 & 78.60\% & 22.806 \\
%\cmidrule[0.5pt]{2-6}
\cline{2-6}
& BCOP & 93.79\% & 237.068 & 73.38\%  & 235.367 \\
%\cline{2-6}
& \textbf{\methodabv} & \textbf{94.44\%} & \textbf{170.864} & \textbf{75.52\%} & \textbf{164.178} \\
\hline
\multirow{2}{*}{Resnet-50} & Standard & 95.47\% & 38.834 & 78.11\% & 37.454 \\
\cline{2-6}
%& BCOP* & \_ & \_ & \_  & \_ \\
%\cline{2-6}
& \textbf{\methodabv} & \textbf{94.68\%} & \textbf{584.762} & \textbf{77.95\%} & \textbf{597.297} \\
\hline
\end{tabular}
\caption{Results for standard accuracy. For Resnet-50, we observe OOM (Out Of Memory) error when using BCOP. }
\label{table:standard_acc}
\end{table*}

\begin{table*}[ht!]
\centering
\renewcommand{\arraystretch}{1.2} 
\begin{tabular}{ p{1.5cm}  p{2cm} | p{1.5cm}  p{1.5cm}  p{1.5cm} | p{1.5cm}  p{1.5cm}  p{1.5cm} }
\hline
\multirow{3}{*}{\textbf{Model}} & \multirow{3}{*}{\textbf{Conv. Type}} & \multicolumn{3}{c|}{\textbf{CIFAR-10}} &  \multicolumn{3}{c}{\textbf{CIFAR-100}} \\
% \cline{3-8}
  &  & \textbf{Standard Accuracy} & \textbf{Robust Accuracy} & \textbf{Time per epoch (s)} & \textbf{Standard Accuracy} & \textbf{Robust Accuracy} & \textbf{Time per epoch (s)}\\
\hline
\multirow{3}{*}{Resnet-18} & Standard & 83.05\% & 44.39\% & 28.139 & 59.87\% & 22.78\% & 28.147 \\
\cline{2-8}
 & BCOP & 79.26\% & 34.85\% & 264.694 & 54.80\% & 16.00\% & 252.868 \\
%\cline{2-8}
& \textbf{\methodabv} & \textbf{82.24\%} & \textbf{43.73\%} & \textbf{203.860} & \textbf{58.95\%} & \textbf{22.65\%} & \textbf{199.188} \\
\hline
\end{tabular}
\caption{Results for empirical robustness against adversarial examples ($l_{\infty}$ perturbation radius of 8/255).}
\label{table:emp_robust_acc}
\end{table*}

\subsection{Standard Training}
For standard training, we perform experiments using Resnet-18, Resnet-34 and Resnet-50 architectures on CIFAR-10 and CIFAR-100 datasets. Results are presented in Table \ref{table:standard_acc}. We again observe that \methodabv\  achieves higher standard accuracy than BCOP on different architectures and datasets while requiring significantly less time to train. For Resnet-50, the performance of \methodabv\ almost matches that of standard convolution layers while BCOP results in an Out Of Memory (OOM) error. However, for Resnet-18 and Resnet-34, the difference is not as significant as the one observed for \archname\ architectures in Table \ref{table:provable_robust_acc}. We conjecture that this is because the residual connections allows the gradient to flow relatively freely compared to being restricted to flow through the convolution layers in \archname\ architectures. 

\subsection{Adversarial Training}
%To evaluate the effectiveness of \methodabv\ for empirical robustness, we use a stronger threat model (than the $l_{2}$ used for evaluating provable robustness).
For adversarial training, we use a threat model with an $l_{\infty}$ attack radius of $8/255$.  Note that we use the $l_{\infty}$ threat model (instead of $l_{2}$) because it is known to be a stronger adversarial threat model for evaluating empirical robustness \cite{madry2018towards}. For training, we use the FGSM variant by \citet{Wong2020Fast}. For evaluation, we use 50 iterations of PGD with step size of $2/255$ and 10 random restarts. Results are presented in Table \ref{table:emp_robust_acc}. We observe that for Resnet-18 architecture and on both CIFAR-10 and CIFAR-100 datasets, \methodabv\ results in significantly improved standard and empirical robust accuracy compared to BCOP while requiring significantly less time to train. The performance of \methodabv\ comes close to the performance of a standard convolution layer with the difference being less than $1\%$ for both standard and robust accuracy on both the datasets. %\SF{one may ask a question why our provable defense is against $\ell_2$ while our adv training is against $\ell_{\infty}$. We should say sth about it.}.

\section{Discussion and Future work}
In this work, we design a new orthogonal convolution layer by first constructing a Skew-Symmetric convolution filter and then applying the convolution exponential \cite{Hoogeboom2020TheCE} to the filter. We also derive provable guarantees on the approximation of the exponential using a finite number of terms. Our method achieves significantly higher accuracy than BCOP for various network architectures and datasets under standard, adversarial and provably robust training setups while requiring less training time per epoch. 
We suggest the following directions for future research:

\textbf{Reducing the evaluation time}: While \methodabv\ requires less time to train than BCOP, it requires more time for evaluation because the convolution filter needs to be applied multiple times to approximate the orthogonal matrix with the desired error. In contrast, BCOP constructs an orthogonal convolution filter that needs to be applied only once during evaluation. From Theorem \ref{thm:approx_error}, we know that we can reduce the number of terms required to achieve the desired approximation error by reducing the Jacobian norm $\|\bJ\|_{2}$. Training approaches such as spectral norm regularization \cite{boundsingular} and singular value clipping \cite{sedghi2018singular} can be useful to further lower $\|\bJ\|_{2}$ and thus reduce the evaluation time. 

% Because we use Spectral Normalization (SN) to bound the Jacobian norm, a major drawback of our approach is that the convolution filter after applying SN necessarily has $\|\bJ\|_{2} \geq 1$ \cite{cisseparseval2017}.

%for smaller Jacobian norm , the number of terms required to achieve the desired approximation error will be lesser. 

%{\color{red} } \SF{this part is hard to understand. How is it related to the complexity of the exponential computation? I think you mean that if the norm is small we can keep fewer terms etc. but from the text, this message is not clear. }
%Using Theorem \ref{thm:appendix_approx_error}, we know that for smaller norm of the Jacobian, the number of terms required to approximate with desired error will be small. Because of this reason, 

\textbf{Complete Set of \methodabv\ convolutions}: While Theorem \ref{thm:approx_cyclic_real} suggests that the complete set of \methodabv\ convolutions can be constructed from a subset of Skew-Symmetric matrices $\bB$ that satisfy (a) $\|\bB\|_{2} \leq \pi$ and (b) $\exp(\bB) = \exp(\bA)$ where $\bA$ is the Jacobian of some Skew-Symmetric convolution filter, it is an open question how to efficiently parametrize this subset for training Lipschitz convolutional neural networks. This remains an interesting problem for future research.

\section{Acknowledgements}
This project was supported in part by NSF CAREER AWARD 1942230, HR001119S0026, HR00112090132, NIST 60NANB20D134  and Simons Fellowship on "Foundations of Deep Learning."

% \textbf{Further improvement of training time and memory usage}: In our current implementation of \methodabv\ convolution, the backward pass is computed using the computation graph from the forward pass. We believe that the training time can be further reduced using a manual implementation of the backward pass using the following property:
% \begin{align*}
%      \exp(-\bA) = \left(\exp(\bA)\right)^{T}\quad \text{where}\ \bA \in \mathbb{R}^{n \times n},\ \bA = - \bA^{T} 
%  \end{align*}
%  Since the backward pass through a linear layer is equal to the transpose of the forward pass, the above property implies that the backward through convolution exponential of $\bA$ can be computed using the convolution exponential of $-\bA$ when $\bA$ is a real Skew-Symmetric matrix. 

%\section{Conclusion}


\bibliography{example_paper}
\bibliographystyle{icml2021}

\clearpage

\appendix

{\Large \bf Appendix}

\section{Notation}\label{sec:appendix_notation}
For $n \in \mathbb{N}$, we use $[n]$ to denote the set $\{0, \dots , n\}$. For a vector $\bv$, we use $\bv_{j}$ to denote the element in the $j^{th}$ position of the vector. We use $\bA_{j,:}$ and $\bA_{:,k}$ to denote the $j^{th}$ row and $k^{th}$ column of the matrix $\bA$ respectively. We assume both $\bA_{j,:}$, $\bA_{:,k}$ to be column vectors (thus $\bA_{j,:}$ is the transpose of $j^{th}$ row of $\bA$). $\bA_{j,k}$ denotes the element in $j^{th}$ row and $k^{th}$ column of $\bA$. $\bA_{j,:k}$ and $\bA_{:j,k}$ denote the vectors containing the first $k$ elements of the $j^{th}$ row and first $j$ elements of $k^{th}$ column, respectively. $\bA_{:j,:k}$ denotes the matrix containing the first $j$ rows and $k$ columns of $\bA$. The same rules can be directly extended to higher order tensors. We use bold zero i.e $\mathbf{0}$ to denote the matrix (or tensor) consisting of zero at all elements, $\bI_{n}$ to denote the identity matrix of size $n \times n$. We use $\mathbb{C}$ to denote the field of complex numbers and $\mathbb{R}$ for real numbers. For a scalar $a \in \mathbb{C}$, $\overline{a}$ denotes its complex conjugate. For a vector $\bv$ or matrix (or tensor) $\bA$, $\overline{\bv}$ or $\overline{\bA}$ denotes the element-wise complex conjugate. For $\bA \in \mathbb{C}^{m \times n}$, $\bA^{H}$ denotes the hermitian transpose i.e  $\bA^{H} = \overline{\bA^{T}}$. For a scalar $a \in \mathbb{C}$, $\operatorname{Re}(a)$, $\operatorname{Im}(a)$ and $|a|$ denote the real part, imaginary part and modulus of $a$ respectively. We use $[a,b)$ where $a,b \in \mathbb{C}$ to denote the set consisting of complex scalars on the line connecting $a$ and $b$ (including $a$, but excluding $b$). $\bA \otimes \bB$ denotes the kronecker product between matrices $\bA$ and $\bB$. We use $\iota$ to denote \textit{iota} (i.e $\iota^{2}=-1$).

For a matrix $\bA \in \mathbb{C}^{q \times r}$ and a tensor $\bB \in \mathbb{C}^{p \times q \times r}$, $\overrightarrow{\bA}$ denotes the vector constructed by stacking the rows of $\bA$ and $ \overrightarrow{\bB}$ by stacking the vectors $\overrightarrow{\bB_{j,:,:}},\ j \in [p-1]$ so that: 
\begin{align*}
&\left(\overrightarrow{\bA}\right)^{T} = \begin{bmatrix}
\bA_{0,:}^{T}\ ,\ \bA_{1,:}^{T}\ ,\ \hdots\ ,\ \bA_{q-1,:}^{T}\end{bmatrix} \\
&\left(\overrightarrow{\bB}\right)^{T} = \begin{bmatrix}
\left(\overrightarrow{\bB_{0,:,:}}\right)^{T}\ ,\ \left(\overrightarrow{\bB_{1,:,:}}\right)^{T}\ ,\ \hdots\ ,\ \left(\overrightarrow{\bB_{p-1,:,:}}\right)^{T} \end{bmatrix}
\end{align*}

For a $2$D convolution filter, $\bL \in \mathbb{C}^{p \times q \times r \times s}$, we define the tensor $\convtransposeoperator(\bL) \in \mathbb{C}^{q \times p \times r \times s}$ as follows:
\begin{align}
 [\mathrm{conv\_transpose}(\bL)]_{i,j,k,l} = \overline{[\bL]}_{j,i,r-1-k,s-1-l}  \label{eq:conv2d_h_appendix}
\end{align}
Note that this is very different from the usual matrix transpose. Given an input $\bX \in \mathbb{C}^{q \times n \times n}$, we use $\bL \star \bX \in \mathbb{C}^{p \times n \times n}$ to denote the convolution of filter $\bL$ with $\bX$. The notation $\bL \star^{i} \bX \triangleq \bL \star^{i-1}\left(\bL \star \bX\right)$. Unless specified otherwise, we assume zero padding and stride 1 in each direction.

\section{Proofs}\label{sec:proofs}
\subsection{Proof of Theorem \ref{thm:2d_kernel}}\label{subsec:proofs_appendix_2d_kernel}
\begin{thm}
	Consider a convolution filter $\bL \in \mathbb{C}^{m \times m \times (2p+1) \times (2q+1)}$ applied to an input $\bX \in \mathbb{C}^{m \times n \times n}$ that results in output $\bY = \bL \star \bX  \in \mathbb{C}^{m \times n \times n}$. Let $\bJ$ be the jacobian of $\overrightarrow{\bY}$ with respect to $\overrightarrow{\bX}$ , then the jacobian for convolution with the filter $\convtransposeoperator(\bL)$ is equal to $\bJ^{H}$.
\end{thm}
\begin{proof}
We first prove the above result assuming $m=1$.\\
\textbf{Assuming} $\mathbf{m=1}$:\\
We know that $\bJ$ is a doubly toeplitz matrix of size $n^{2} \times n^{2}$:
\[\bJ = \begin{bmatrix}
	\bJ^{(0)} & \bJ^{(-1)} & \cdots & \bJ^{(-p)} & 0 \\
	\bJ^{(1)} & \bJ^{(0)} & \bJ^{(-1)} & \ddots & \ddots \\
	\vdots & \bJ^{(1)} & \bJ^{(0)} & \ddots & \ddots\\
	\bJ^{(p)} & \ddots & \ddots & \ddots & \bJ^{(-1)}\\
	0 & \ddots & \ddots & \bJ^{(1)} & \bJ^{(0)}
	\end{bmatrix}\]
In the above equation, each $\bJ^{(i)}$ is a toeplitz matrix of size $n \times n$. Define $\bP^{(k)}$ as a $n \times n$ matrix with $\bP^{(k)}_{i,j}=1$ if $i-j=k$ and $0$ otherwise. Thus $\bJ$ can be written as:
$$\bJ = \sum_{i=-p}^{p} \bP^{(i)} \otimes \bJ^{(i)} $$
Since each matrix $\bJ^{(i)}$ is a toeplitz matrix, it can be written as follows. Because the first two dimensions of filter $\bL$ are of size $1$, we index $\bL$ using only the last two indices:
$$\bJ^{(i)} = \sum_{j=-q}^{q} \bL_{p+i, q+j} \bP^{(j)}  $$
Thus, $\bJ$ can be written as:
$$\bJ = \sum_{i=-p}^{p} \sum_{j=-q}^{q} \bL_{p+i, q+j} \left(\bP^{(i)} \otimes \bP^{(j)}\right) $$
Thus, $\bJ^{H}$ can be written as:
\begin{align*}
&\bJ^{H} = \sum_{i=-p}^{p} \sum_{j=-q}^{q} \overline{\bL_{p+i, q+j}} \left(\bP^{(i)} \otimes \bP^{(j)}\right)^{T} \\
&\bJ^{H} = \sum_{i=-p}^{p} \sum_{j=-q}^{q} \overline{\bL_{p+i, q+j}} \left(\bP^{(i)}\right)^{T} \otimes \left(\bP^{(j)}\right)^{T} \\
&\bJ^{H} = \sum_{i=-p}^{p} \sum_{j=-q}^{q} \overline{\bL_{p+i, q+j}} \left(\bP^{(-i)} \otimes \bP^{(-j)}\right) \\
&\bJ^{H} = \sum_{i=-p}^{p} \sum_{j=-q}^{q} \overline{\bL_{p-i, q-j}} \left(\bP^{(i)} \otimes \bP^{(j)}\right) 
\end{align*}
Thus $\bJ^{H}$ corresponds to the jacobian of the convolution filter flipped along the third, fourth axis and each individual element conjugated.\\

Next, we prove the result when $m>1$.\\
\textbf{Assuming} $\mathbf{m > 1}$:\\
We know that $\bJ$ is a matrix of size $mn^{2} \times mn^{2}$. Let $\bJ^{(i,j)}$ denote the block of size $n^{2} \times n^{2}$ as follows:
$$ \bJ^{(i,j)} = \bJ_{in^{2}:(i+1)n^{2}, jn^{2}:(j+1)n^{2}} $$
Note that $\bJ^{(i,j)}$ is the jacobian of convolution with $1 \times 1$ filter $\bL_{i:i+1,j:j+1,:,:}$. Now consider the $(i,j)^{th}$ block of $\bJ^{H}$. Using definition of conjugate transpose (i.e ${H}$ operator):
\begin{align}
 \big(\bJ^H\big)^{(i,j)} = \big(\bJ^{(j,i)}\big)^{H} \label{eq:j_block_2d}
\end{align}
Consider the $1 \times 1$ filter at the $(i,j)^{th}$ index in $\convtransposeoperator(\bL)$. By the definition of $\convtransposeoperator$ operator, we have:
\begin{align}
    &\left[\convtransposeoperator(\bL)\right]_{i:i+1,j:j+1,:,:} \nonumber\\
    &= \convtransposeoperator(\bL_{j:j+1,i:i+1,:,:}) \label{eq:conv_block_2d}
\end{align}
Using equations \eqref{eq:j_block_2d} and \eqref{eq:conv_block_2d} and the proof for the case $\bm=1$, we have the desired proof.

\end{proof}

\subsection{Proof of Theorem \ref{thm:main_theorem}}\label{subsec:proofs_appendix_main_theorem}
\begin{thm}
	Consider a convolution filter $\bL \in \mathbb{C}^{m \times m \times (2p+1) \times (2q+1)}$. Given an input $\bX \in \mathbb{C}^{m \times n \times n}$, output $\bY = \bL \star \bX \in \mathbb{C}^{m \times n \times n}$. The jacobian of $\overrightarrow{\bY}$ with respect to $\overrightarrow{\bX}$ (call it $\bJ$) will be a matrix of size $n^{2}m \times n^{2}m$. $\bJ$ is a skew hermitian matrix if and only if:
	$$\bL = \bM - \convtransposeoperator(\bM) $$
	for some filter $\bM \in \mathbb{C}^{m \times m \times (2p+1) \times (2q+1)}$:
\end{thm}
\begin{proof}
We first prove that if $\bJ$ is a skew-hermitian matrix, then:
$$\bL = \bM - \convtransposeoperator(\bM) $$
Let $\bJ^{(i,j)}$ denote the block of size $n^{2} \times n^{2}$ as follows:
$$ \bJ^{(i,j)} = \bJ_{in^{2}:(i+1)n^{2}, jn^{2}:(j+1)n^{2}} $$
so that $\bJ$ can be written in terms of the blocks $\bJ^{(i,j)}$:
\[\bJ = \begin{bmatrix}
	\bJ^{(0,0)}&\bJ^{(0,1)}&\cdots &\bJ^{(0,m-1)} \\
	\bJ^{(1,0)}&\bJ^{(1,1)}&\cdots &\bJ^{(1,m-1)} \\
	\vdots & \vdots & \ddots & \vdots\\
	\bJ^{(m-1,0)}&\bJ^{(m-1,1)}&\cdots &\bJ^{(m-1,m-1)}
	\end{bmatrix}\]

Since $\bJ$ is skew-hermitian, we have:
$$ \bJ^{(i,j)} = -\left(\bJ^{(j,i)}\right)^{H}, \quad \forall\ i,\ j \in [m-1] $$
It is readily observed that $\bJ^{(i,j)}$ corresponds to the jacobian of convolution with $1 \times 1$ filter $\bL_{i:i+1,j:j+1,:,:}$. For some given filter $\bA$, we use $\bA^{(i,j)}$ to denote the $1 \times 1$ filter $\bA_{i:i+1,j:j+1,:,:}$ for simplicity. Thus, the above equation can be rewritten as:
\begin{align}
\bL^{(i,j)} = -\convtransposeoperator\left(\bL^{(j,i)}\right), \quad \forall\ i,\ j \in [m-1] \label{eq:L_i_j_h_equal}
\end{align}
Now construct a filter $\bM$ such that for $i \neq j$:
\begin{align}
&\bM^{(i,j)} = \begin{cases}
\bL^{(i,j)}, \qquad i < j\\
\mathbf{0}, \qquad i > j \\
\end{cases} \label{eq:M_i_j_neq}
\end{align}
For $i=j$, $\bM$ is given as follows:
\begin{align}
&\bM^{(i,i)}_{r,s} = \begin{cases}
\bL^{(i,i)}_{r,s}, \qquad r \leq p-1\\
\bL^{(i,i)}_{r,s}, \qquad r = p,\ s \leq q-1 \\
0.5 \times \bL^{(i,i)}_{r,s}, \qquad r = p,\ s = q \\
0, \qquad otherwise
\end{cases} \label{eq:M_i_i}
\end{align}
Next, our goal is to show that:
\begin{align*}
    \bL = \bM - \convtransposeoperator(\bM)
\end{align*}
Now by the definition of $\convtransposeoperator$, we have:
\begin{align}
& [\bM - \convtransposeoperator(\bM)]^{(i,j)} \nonumber \\
& = \bM^{(i,j)} - \left[\convtransposeoperator(\bM)\right]^{(i,j)} \nonumber \\
& = \bM^{(i,j)} - \convtransposeoperator\left(\bM^{(j,i)}\right) \label{eq:M_diff_simple}
\end{align}
\textbf{Case 1:} For $i < j$, using equations \eqref{eq:L_i_j_h_equal} and \eqref{eq:M_i_j_neq}:
$$\bM^{(i,j)} - \convtransposeoperator\left(\bM^{(j,i)}\right) = \bM^{(i,j)} = \bL^{(i,j)}$$
\textbf{Case 2:} For $i > j$, using equations \eqref{eq:L_i_j_h_equal} and \eqref{eq:M_i_j_neq}:
\begin{align*}
&\bM^{(i,j)} - \convtransposeoperator\left(\bM^{(j,i)}\right) \\
&= - \convtransposeoperator\left(\bM^{(j,i)}\right) \\
&= - \convtransposeoperator\left(\bL^{(j,i)}\right) = \bL^{(i,j)}
\end{align*}
\textbf{Case 3:} For $i = j$, we further simplify equation \eqref{eq:M_diff_simple}:
\begin{align}
&\bM^{(i,i)}_{r,s} - \bigg[\convtransposeoperator\left(\bM^{(i,i)}\right)\bigg]_{r,s} \nonumber \\
&= \bM^{(i,i)}_{r,s} - \overline{\bM^{(i,i)}_{2p-r,2q-s}} \label{eq:M_diff_simple_element}
\end{align}
\textbf{Subcase 3(a):} For ($r \leq p-1$) or ($r = p,\ s \leq q-1$), we have:
$$ \bM^{(i,i)}_{2p-r,2q-s} = 0 $$
Thus for ($r \leq p-1$) or ($r = p,\ s \leq q-1$): equation \eqref{eq:M_diff_simple_element} simplifies to $\bM^{(i,i)}_{r,s}$. The result follows trivially from the very definition of $\bM^{(i,i)}_{r,s}$, i.e equation \eqref{eq:M_i_i}.\\ \\
\textbf{Subcase 3(b):} For ($r \geq p+1$) or ($r = p ,\ s \geq q+1$), we have: $$\bM^{(i,i)}_{r,s} = 0$$
Thus, equation \eqref{eq:M_diff_simple_element} simplifies to:
\begin{align*}
& \bM^{(i,i)}_{r,s} - \overline{\bM^{(i,i)}_{2p-r,2q-s}} = - \overline{\bM^{(i,i)}_{2p-r,2q-s}} 
\end{align*}
Since ($r \geq p+1$) or ($r = p,\ s \geq q+1$), we have: ($2p-r \leq p-1$) or ($2p-r=p,\ 2q - s \leq q-1$) respectively. Thus using equation \eqref{eq:M_i_i}, we have:
\begin{align*}
& - \overline{\bM^{(i,i)}_{2p-r,2q-s}} = - \overline{\bL^{(i,i)}_{2p-r,2q-s}} 
\end{align*}
Since $\bL^{(i,i)}$ is a skew-hermitian filter, we have from Theorem \ref{thm:2d_kernel}:
\begin{align*}
& \bL^{(i,i)}_{r, s} = - \overline{\bL^{(i,i)}_{2p-r,2q-s}}
\end{align*}
Thus in this subcase, equation \eqref{eq:M_diff_simple_element} simplifies to $\bL^{(i,i)}_{r, s}$ again.\\
\textbf{Subcase 3(c):} For $r = p ,\ s = q$, since $\bL^{(i,i)}$ is a skew-hermitian filter, we have:
\begin{align*}
& \bL^{(i,i)}_{p, q} = - \overline{\bL^{(i,i)}_{p,q}} \\
& \bL^{(i,i)}_{p, q} + \overline{\bL^{(i,i)}_{p,q}} = 0
\end{align*}
Thus, $\bL^{(i,i)}_{p, q}$ is a purely imaginary number. In this subcase
\begin{align*}
& \bM^{(i,i)}_{r,s} - \overline{\bM^{(i,i)}_{2p-r,2q-s}} \\
& = \bM^{(i,i)}_{p,q} - \overline{\bM^{(i,i)}_{p,q}} = 2\bM^{(i,i)}_{p,q}
\end{align*}
Using equation \eqref{eq:M_i_i} , we have:
\begin{align*}
2\bM^{(i,i)}_{p,q} = \bL^{(i,i)}_{p,q}
\end{align*}
Thus, we get:
\begin{align*}
&\bM^{(i,i)}_{r,s} - \bigg[\convtransposeoperator\left(\bM^{(i,i)}\right)\bigg]_{r,s} = \bL^{(i,i)}_{r, s} 
\end{align*}

Thus we have established: $\bL = \bM - \convtransposeoperator(\bM)$. Note that the opposite direction of the if and only if statement follows trivially from the above proof.
\end{proof}

\subsection{Proof of Theorem \ref{thm:approx_error}}\label{subsec:proofs_appendix_approx_error}
\begin{thm}
    \begin{enumerate}[label=(\alph*)]
    \item For a scalar $\lambda \in \mathbb{C}$ with $\operatorname{Re}(\lambda)=0$, the error between $\exp(\lambda)$ and approximation $p_{k}(\lambda)$ given below can be bounded as follows:
    \begin{align}
	 &\exp(\lambda) = \sum_{i=0}^{\infty} \frac{\lambda^{k}}{i!},\qquad p_{k}(\lambda) = \sum_{i=0}^{k-1} \frac{\lambda^{i}}{i!} \label{eq:pk_lambda} \\
     &\big| \exp(\lambda) - p_{k}(\lambda) \big| \leq \frac{|\lambda|^{k}}{k!},\qquad \forall\ \lambda: \operatorname{Re}(\lambda)=0 \nonumber
    \end{align}
    \item For a skew-hermitian matrix $\bA$, the error between $\exp(\bA)$ and the series approximation $\bS_{k}(\bA)$ can be bounded as follows:
    \begin{align*}
    &\exp(\bA) = \sum_{i=0}^{\infty} \frac{\bA^{i}}{i!},\qquad \bS_{k}(\bA) = \sum_{i=0}^{k-1} \frac{\bA^{i}}{i!} \nonumber\\
	&\|\exp(\bA) - \bS_{k}(\bA)\|_{2} \leq \frac{\|\bA\|_{2}^{k}}{k!}\nonumber
	\end{align*}
    \end{enumerate}
\end{thm}
\begin{proof}
Since $\bA$ is skew-hermitian, it is a normal matrix and eigenvectors for distinct eigenvalues must be orthogonal. Let the eigenvalue decomposition of $\bA$ be given as follows:
\begin{align*}
    &\bA = \bU \Lambda \bU^{H} 
\end{align*}
Note that $\Lambda$ is a diagonal matrix, and each element along the diagonal is purely imaginary (since $\bA$ is skew-hermitian). Exponentiating both sides, we get:
\begin{align*}
    &\exp(\bA) = \bU \exp(\Lambda) \bU^{H}
\end{align*}
Thus the error $\bE_{k}(\bA)$ is given by:
\begin{align}
&\bE_{k}(\bA) = \exp(\bA) - \bS_{k}(\bA) \label{eq:error_def} \\
&\bE_{k}(\bA) = \bU \left(\exp(\Lambda) - \bS_{k}(\Lambda)\right) \bU^{H} \nonumber \\
&\|\bE_{k}(\bA)\|_{2} = \|\bU \left(\exp(\Lambda) - \bS_{k}(\Lambda)\right) \bU^{H} \|_{2} \nonumber \\
&\|\bE_{k}(\bA)\|_{2} = \|\left(\exp(\Lambda) - \bS_{k}(\Lambda)\right) \|_{2} \nonumber
\end{align}
Since $\left(\exp(\Lambda) - \bS_{k}(\Lambda)\right)$ is a diagonal matrix, we have:
\begin{align}
&\|\bE_{k}(\bA)\|_{2} = \max_{i} \big|\left(\exp(\Lambda_{i,i}) - p_{k}(\Lambda_{i,i})\right) \big| \label{eq:Ek_error}
\end{align}
Let $\lambda$ be an arbitrary element along the diagonal of $\Lambda$ i.e $\lambda = \Lambda_{i,i}$ for some $i$. First note that:
\begin{align*}
&\lambda \int^{1}_{0} \{\exp(t\lambda) - p_{k}(t\lambda) \}dt \\
&=  \int^{1}_{0} \{\exp(t\lambda) - p_{k}(t\lambda) \}\lambda dt 
\end{align*}
Substituting $u = \lambda t$, we have:
\begin{align*}
&=  \int^{\lambda}_{0} \{\exp(u) - p_{k}(u) \} du \\
&=  \int^{\lambda}_{0} \exp(u)du - \int^{\lambda}_{0}p_{k}(u)du \\
&=  \exp(\lambda) - 1 - \int^{\lambda}_{0}p_{k}(u)du \\
\end{align*}
Substituting $p_{k}(u)$ using equation \eqref{eq:pk_lambda}:
\begin{align*}
&=  \exp(\lambda) - 1 - \int^{\lambda}_{0} \sum_{i=0}^{k-1} \frac{u^{i}}{i!} du \\
&=  \exp(\lambda) - 1 - \sum_{i=0}^{k-1} \int^{\lambda}_{0} \frac{u^{i}}{i!} du \\
&=  \exp(\lambda) - 1 - \sum_{i=0}^{k-1} \frac{u^{i+1}}{(i+1)!} \bigg|^{\lambda}_{0} \\
&=  \exp(\lambda) - 1 - \sum_{i=0}^{k-1} \frac{\lambda^{i+1}}{(i+1)!} \\
&=  \exp(\lambda) - 1 - \sum_{i=1}^{k} \frac{\lambda^{i}}{i!} \\
&=  \exp(\lambda) - \sum_{i=0}^{k+1} \frac{\lambda^{i}}{i!} \\
&=  \exp(\lambda) - p_{k+1}(\lambda)
\end{align*}
This gives the following result:
\begin{align}
\exp(\lambda) - p_{k+1}(\lambda) = \lambda \int^{1}_{0} \bigg(\exp(t\lambda) - p_{k}(t\lambda) \bigg)dt \label{eq:exp_induction}
\end{align}

We shall now prove the main result using induction and equation \eqref{eq:exp_induction}: \\
\textbf{Base case}: \\
Use $k=0$ and the convention that $p_{0}(\lambda)=0$. We know that $p_{1}(\lambda)=1$.
\begin{align*}
&\exp(\lambda) - p_{1}(\lambda) = \lambda \int^{1}_{0} \bigg(\exp(t\lambda) - p_{0}(t\lambda) \bigg)dt \\
&\big|\exp(\lambda) - 1\big| = \big|\lambda \int^{1}_{0} \exp(t\lambda)dt \big| 
\end{align*}
Since $\lambda$ is purely imaginary and $t$ is purely real, we have $|\exp(t\lambda)|=1$:
\begin{align*}
&\big|\exp(\lambda) - 1\big| \leq \big|\lambda\big| \int^{1}_{0} \big| \exp(t\lambda) \big| dt = \big|\lambda\big| \int^{1}_{0} 1 dt \\
&\big|\exp(\lambda) - 1\big| \leq \big|\lambda\big| 
\end{align*}

\textbf{Induction step}: \\
Assuming this holds for all $k$ i.e: 
\begin{align}
\big|\exp(\lambda) - p_{k}(\lambda)\big| \leq \frac{|\lambda|^{k}}{k!} \label{eq:induction_assumption}
\end{align}
Now let us consider $\big|\exp(\lambda) - p_{k+1}(\lambda)\big|$:
\begin{align*}
&\big|\exp(\lambda) - p_{k+1}(\lambda)\big| \leq \bigg|\lambda  \int^{1}_{0} \bigg(\exp(t\lambda) - p_{k}(t\lambda) \bigg)dt \bigg| \\
&\big|\exp(\lambda) - p_{k+1}(\lambda)\big| \leq \big|\lambda \big|  \int^{1}_{0} \bigg|\bigg(\exp(t\lambda) - p_{k}(t\lambda) \bigg)\bigg| dt  
\end{align*}
Using equation \eqref{eq:induction_assumption}, we have:
\begin{align*}
&\big|\exp(\lambda) - p_{k+1}(\lambda)\big| \leq \big|\lambda \big| \int^{1}_{0}\frac{|t\lambda|^{k}}{k!}dt \\
&\big|\exp(\lambda) - p_{k+1}(\lambda)\big| \leq \big|\lambda \big|^{k+1} \int^{1}_{0}\frac{|t|^{k}}{k!}dt \\
&\big|\exp(\lambda) - p_{k+1}(\lambda)\big| \leq \frac{\big|\lambda \big|^{k+1}}{(k+1)!} 
\end{align*}
This proves $(a)$.\\
Since $\lambda$ is an arbitrary element along the diagonal of eigenvalue matrix $\Lambda$, using equations \eqref{eq:error_def} and \eqref{eq:Ek_error} we have:
\begin{align}
&\|\exp(\bA) - \bS_{k}(\bA)\|_{2} = \max_{i} \big|\exp(\Lambda_{i,i}) - p_{k}(\Lambda_{i,i}) \big| \nonumber \\ 
&\|\exp(\bA) - \bS_{k}(\bA)\|_{2} \leq \max_{i} \frac{\big|\Lambda_{i,i} \big|^{k}}{k!} \nonumber \\
&\|\exp(\bA) - \bS_{k}(\bA)\|_{2} \leq \frac{1}{k!} \max_{i} \big|\Lambda_{i,i} \big|^{k} \label{eq:almost_final_error}
\end{align}
Since $\bA$ is skew-hermitian, it is a normal matrix and singular values are equal to the magnitude of eigenvalues. Thus we have from equation \eqref{eq:almost_final_error}:
\begin{align*}
&\max_{i} \big|\Lambda_{i,i} \big| = \left\|\Lambda \right\|_{2} = \left\|\bA \right\|_{2} \\
&\|\exp(\bA) - \bS_{k}(\bA)\|_{2} \leq \frac{\left\|\bA \right\|^{k}_{2}}{k!} \nonumber
\end{align*}
This proves $(b)$.
\end{proof}

\subsection{Proof of Theorem \ref{thm:approx_cyclic_real}}\label{subsec:proofs_appendix_approx_cyclic_real}
\begin{thm}
    Given a real skew-symmetric matrix $\bA \in \mathbb{R}^{n \times n}$, we can construct a real skew-symmetric matrix $\bB \in \mathbb{R}^{n \times n}$ such that $\bB$ satisfies: (a) $\exp(\bA)=\exp(\bB)$ and (b) $\|\bB\|_{2} \leq \pi$.
\end{thm}
\begin{proof}
We know that for eigenvalues of real symmetric matrices are purely imaginary and come in pairs: $\lambda_{1}\iota,\ -\lambda_{1}\iota, \lambda_{2}\iota, -\lambda_{2}\iota$ where each $\lambda_{i}$ is real. When $n$ is an odd integer, $0$ is an eigenvalue. Additionally, we know that a real skew symmetric matrix can be expressed in a block diagonal form as follows:
\begin{align}
    \bA = \bQ\Sigma\bQ^{T} \label{eq:A_block_diag}
\end{align}
Here $\bQ$ is a real orthogonal matrix and $\Sigma$ is a block diagonal matrix defined as follows:
\begin{align}
    \Sigma_{2i:2i+2,2i:2i+2} = \begin{bmatrix}
	0 & \lambda_{i}  \\
	-\lambda_{i} & 0 \\
	\end{bmatrix},\qquad 0 \leq i < \left \lfloor{\frac{n}{2}}\right \rfloor \label{eq:2_2_real_block}
\end{align}
In the above equation, $\lambda_{i} \in \mathbb{R}$ and $\pm\lambda_{i}\iota$ are the eigenvalues of $\bA$. When $n$ is odd, we additionally have:
\begin{align*}
    \Sigma_{n-1,n-1} = 0
\end{align*}
Taking the exponential of both sides of equation \eqref{eq:A_block_diag}:
\begin{align}
    \exp\left(\bA\right) = \bQ\exp\left(\Sigma\right)\bQ^{T} \label{eq:A_block_diag_exp}
\end{align}
We can compute $\exp\left(\Sigma\right)$ by computing the exponential of each $2 \times 2$ block defined in equation \eqref{eq:2_2_real_block}:
\begin{align}
    &\exp\bigg(\begin{bmatrix}
	0 & \lambda_{i}  \\
	-\lambda_{i} & 0 \\
	\end{bmatrix}\bigg) \nonumber \\
	&= \exp\bigg(\lambda_{i}\begin{bmatrix}
	0 & 1  \\
	-1 & 0 \\
	\end{bmatrix}\bigg) = \begin{bmatrix}
	\cos(\lambda_{i}) & -\sin(\lambda_{i})  \\
	\sin(\lambda_{i}) & \cos(\lambda_{i}) \\
	\end{bmatrix} \label{eq:exp_2_2_real_block}
\end{align}
From equation \eqref{eq:exp_2_2_real_block}, we observe each $\lambda_{i}$ can be shifted by integer multiples of $2\pi\iota$ without changing the exponential. \\
For each $\lambda_{i}, i \in [\left \lfloor{n/2}\right \rfloor-1]$, we define a scalar $\mu_{i}$:
\begin{align}
&\mu_{i} = \lambda_{i} + 2\pi k_{i} \iota,\qquad k_{i} \in \mathbb{Z} \label{eq:mu_i_def} \\
&\mu_{i} \in [-\pi \iota, \pi \iota) \label{eq:mu_i_bound}
\end{align}
Construct a new matrix $\bB$ defined as follows:
\begin{align}
 \bB = \bQ \bD \bQ^{T} \label{eq:BQD_def}
\end{align} 
The matrix $\bD$ in equation \eqref{eq:BQD_def} is defined as follows:
\begin{align}
    \bD_{2i:2i+2,2i:2i+2} = \begin{bmatrix}
	0 & \mu_{i}  \\
	-\mu_{i} & 0 \\
	\end{bmatrix},\qquad 0 \leq i < \left \lfloor{\frac{n}{2}}\right \rfloor \label{eq:D_2_2_real_block}
\end{align}
Let us verify that $\bB$ satisfies the following properties.\\
Using equations \eqref{eq:exp_2_2_real_block}, \eqref{eq:mu_i_def} and \eqref{eq:D_2_2_real_block}, we know that:
\begin{align*}
    &\exp(\bD) =  \exp(\Lambda)
\end{align*}
This results in the following set of equations:
\begin{align*}
    &\exp(\bB) = \bQ\exp(\bD)\bQ^{T} \\
    &\exp(\bB) = \bQ\exp(\Lambda)\bQ^{T} = \exp(\bA)
\end{align*}
Using equations \eqref{eq:mu_i_bound} and \eqref{eq:D_2_2_real_block}, we have:
\begin{align*}
    &\|\bB\|_{2} = \|\bQ\bD\bQ^{T}\|_{2} = \|\bD\|_{2} \\
    &\|\bD\|_{2} \leq \pi
\end{align*}
Note that $\bB$ is a product of $3$ real matrices $\bQ$, $\bD$ and $\bQ^{T}$ and hence $\bB$ is real. Moreover, since $\bD$ is skew symmetric, $\bB$ is skew symmetric.
\end{proof}

\subsection{Proof of Theorem \ref{thm:appendix_approx_cyclic_complex}}\label{subsec:proofs_appendix_approx_cyclic_complex}
\begin{theorem}\label{thm:appendix_approx_cyclic_complex}
    Given a skew-hermitian matrix $\bA$, we can construct a skew-hermitian matrix $\bB$ by adding integer multiples of $2\pi \iota$ to eigenvalues of $\bA$ such that $\bB$ satisfies: (a) $\exp(\bA)=\exp(\bB)$ and (b) $\|\bB\|_{2} \leq \pi$.
\end{theorem}
\begin{proof}
Let the eigenvalue decomposition of $\bA$ be given:
$$ \bA = \bU \Lambda \bU^{H} $$
Let $\lambda_{j}$ be some eigenvalue of $\bA$ such that:
$$ \lambda_{j} = \Lambda_{j,j} $$
Construct a new diagonal matrix $\bD$ of eigenvalues such that:
\begin{align}
&\bD_{j,j} = \lambda_{j} + 2\pi k_{j} \iota,\qquad k_{j} \in \mathbb{Z} \label{eq:D_j_j_def}\\
&\bD_{j,j} \in [-\pi \iota, \pi \iota) \label{eq:D_j_j_bound}
\end{align}
Construct a new matrix $\bB$ defined as follows:
$$ \bB = \bU \bD \bU^{H} $$
Let us verify that $\bB$ satisfies the following properties.\\
Using equation \eqref{eq:D_j_j_def}, we have:
\begin{align*}
    &\exp(\bB) = \bU\exp(\bD)\bU^{H} \\
    &\exp(\bB) = \bU\exp(\Lambda)\bU^{H} = \exp(\bA)
\end{align*}
Using equation \eqref{eq:D_j_j_bound}, we have:
\begin{align*}
    &\|\bB\|_{2} = \|\bU\bD\bU^{H}\|_{2} = \|\bD\|_{2} \\
    &\|\bD\|_{2} = \max_{j} |D_{j,j}| \leq \pi
\end{align*}

\end{proof}

\subsection{Proof of Theorem \ref{thm:appendix_3d_kernel}}\label{subsec:proofs_appendix_3d_kernel}
\begin{theorem}\label{thm:appendix_3d_kernel}
	Consider a convolution filter $\bL \in \mathbb{C}^{m \times m \times (2p+1) \times (2q+1) \times (2r+1)}$ applied to an input $\bX \in \mathbb{C}^{m \times n \times n \times n}$ that results in output $\bY = \bL \star \bX \in \mathbb{C}^{m \times n \times n \times n}$. Let $\bJ$ be the jacobian of $\overrightarrow{\bY}$ with respect to $\overrightarrow{\bX}$ , then the jacobian for convolution with the filter $\convthreedtransposeoperator(\bL)$ is equal to $\bJ^{H}$.
\end{theorem}
\begin{proof}
We first prove the above result assuming $m=1$.\\
\textbf{Assuming} $\mathbf{m=1}$:\\
Because the first two dimensions of filter $\bL$ are of size $1$, we index $\bL$ using only the last two indices. Define $\bP^{(k)}$ as a $n \times n$ matrix with $\bP^{(k)}_{i,j}=1$ if $i-j=k$ and $0$ otherwise. We know that $\bJ$ is a triply toeplitz matrix of size $n^{3} \times n^{3}$ given as follows:
$$\bJ = \sum_{i=-p}^{p} \sum_{j=-q}^{q} \sum_{k=-r}^{r} \bL_{p+i, q+j, r+k} \left(\bP^{(i)} \otimes \bP^{(j)} \otimes \bP^{(k)} \right) $$
Thus, $\bJ^{H}$ can be written as:
\begin{align*}
&\bJ^{H} \\
&=\sum_{i=-p}^{p} \sum_{j=-q}^{q} \sum_{k=-r}^{r} \overline{\bL_{p+i, q+j, r+k}} \left(\bP^{(i)} \otimes \bP^{(j)} \otimes \bP^{(k)} \right)^{T}\\
&=\sum_{i=-p}^{p} \sum_{j=-q}^{q} \sum_{k=-r}^{r} \overline{\bL_{p+i, q+j, r+k}} \bP^{(-i)} \otimes \bP^{(-j)} \otimes \bP^{(-k)} \\
&=\sum_{i=-p}^{p} \sum_{j=-q}^{q} \sum_{k=-r}^{r} \overline{\bL_{p-i, q-j, r-k}} \bP^{(i)} \otimes \bP^{(j)} \otimes \bP^{(k)} 
\end{align*}
Thus $\bJ^{H}$ corresponds to the jacobian of the convolution filter flipped along the third, fourth, fifth axis and each individual element conjugated. \\

Next, we prove the above result when $m > 1$.\\
\textbf{Assuming} $\mathbf{m > 1}$:\\
We know that $\bJ$ is a matrix of size $mn^{3} \times mn^{3}$. Let $\bJ^{(i,j)}$ denote the block of size $n^{3} \times n^{3}$ as follows:
$$ \bJ^{(i,j)} = \bJ_{in^{3}:(i+1)n^{3}, jn^{3}:(j+1)n^{3}} $$
Note that $\bJ^{(i,j)}$ is the jacobian of convolution with $1 \times 1$ filter $\bL_{i:i+1,j:j+1,:,:}$. Now consider the $(i,j)^{th}$ block of $\bJ^{H}$. Using definition of conjugate transpose (i.e ${H}$ operator):
\begin{align}
 \big(\bJ^H\big)^{(i,j)} = \big(\bJ^{(j,i)}\big)^{H} \label{eq:j_block_3d}
\end{align}
Consider the $1 \times 1$ filter at the $(i,j)^{th}$ index in $\convthreedtransposeoperator(\bL)$. By the definition of $\convthreedtransposeoperator$ operator, we have:
\begin{align}
    &\left[\convthreedtransposeoperator(\bL)\right]_{i:i+1,j:j+1,:,:} \nonumber \\
    &= \convthreedtransposeoperator(\bL_{j:j+1,i:i+1,:,:}) \label{eq:conv_block_3d}
\end{align}
Using equations \eqref{eq:j_block_3d} and \eqref{eq:conv_block_3d} and the proof for the case $\bm=1$, we have the desired proof.

\end{proof}

\subsection{Proof of Theorem \ref{thm:appendix_main_theorem_3d}}\label{subsec:proofs_appendix_main_theorem_3d}
\begin{theorem}\label{thm:appendix_main_theorem_3d}
	Consider a convolution filter $\bL \in \mathbb{C}^{m \times m \times (2p+1) \times (2q+1) \times (2r+1)}$. Given an input $\bX  \in \mathbb{C}^{m \times n \times n \times n}$, output $\bY = \bL \star \bX \in \mathbb{C}^{m \times n \times n \times n}$. The jacobian of $\overrightarrow{\bY}$ with respect to $\overrightarrow{\bX}$ (call it $\bJ$) will be a matrix of size $n^{3}m \times n^{3}m$. $\bJ$ is a skew hermitian matrix if and only if:
	$$\bL = \bM - \convthreedtransposeoperator(\bM) $$
	for some filter $\bM \in \mathbb{C}^{m \times m \times (2p+1) \times (2q+1) \times (2r+1)}$:
\end{theorem}
\begin{proof}
We first prove that if $\bJ$ is a skew-hermitian matrix, then:
$$\bL = \bM - \convthreedtransposeoperator(\bM) $$
Let $\bJ^{(i,j)}$ denote the block of size $n^{3} \times n^{3}$ as follows:
$$ \bJ^{(i,j)} = \bJ_{in^{3}:(i+1)n^{3}, jn^{3}:(j+1)n^{3}} $$
Since $\bJ$ is skew-hermitian, we have:
$$ \bJ^{(i,j)} = -\left(\bJ^{(j,i)}\right)^{H}, \quad \forall\ i,\ j \in [m-1] $$
It is readily observed that $\bJ^{(i,j)}$ corresponds to the jacobian of convolution with $1 \times 1$ filter $\bL_{i:i+1,j:j+1,:,:,:}$. For some given filter $\bA$, we use $\bA^{(i,j)}$ to denote the $1 \times 1$ filter $\bA_{i:i+1,j:j+1,:,:,:}$ for simplicity. Thus, the above equation can be rewritten as:
\begin{align}
\bL^{(i,j)} = -\convthreedtransposeoperator\left(\bL^{(j,i)}\right), \quad \forall\ i,\ j \in [m-1] \label{eq:L_i_j_h_equal_3d}
\end{align}
Now construct a filter $\bM$ such that for $i \neq j$:
\begin{align}
&\bM^{(i,j)} = \begin{cases}
\bL^{(i,j)}, \qquad i < j\\
\mathbf{0}, \qquad i > j \\
\end{cases} \label{eq:M_i_j_neq_3d}
\end{align}
For $i=j$, $\bM$ is given as follows:
\begin{align}
&\bM^{(i,i)}_{s,t,u} = \begin{cases}
\bL^{(i,i)}_{s,t,u}, \qquad s \leq p-1\\
\bL^{(i,i)}_{s,t,u}, \qquad s = p,\ t \leq q-1 \\
\bL^{(i,i)}_{s,t,u}, \qquad s = p,\ t = q,\ u \leq r-1 \\
0.5 \times \bL^{(i,i)}_{s,t,u}, \qquad s = p,\ t = q,\ u = r \\
0, \qquad otherwise
\end{cases} \label{eq:M_i_i_3d}
\end{align}
Next, our goal is to show that:
\begin{align*}
    \bL = \bM - \convthreedtransposeoperator(\bM)
\end{align*}
Now by the definition of $\convthreedtransposeoperator$, we have:
\begin{align}
& [\bM - \convthreedtransposeoperator(\bM)]^{(i,j)} \nonumber \\
& = \bM^{(i,j)} - \left[\convthreedtransposeoperator(\bM)\right]^{(i,j)} \nonumber \\
& = \bM^{(i,j)} - \convthreedtransposeoperator\left(\bM^{(j,i)}\right) \label{eq:M_diff_simple_3d}
\end{align}
\textbf{Case 1:} For $i < j$, using equations \eqref{eq:L_i_j_h_equal_3d} and \eqref{eq:M_i_j_neq_3d}:
$$\bM^{(i,j)} - \convthreedtransposeoperator\left(\bM^{(j,i)}\right) = \bM^{(i,j)} = \bL^{(i,j)}$$
\textbf{Case 2:} For $i > j$, using equations \eqref{eq:L_i_j_h_equal_3d} and \eqref{eq:M_i_j_neq_3d}:
\begin{align*}
&\bM^{(i,j)} - \convthreedtransposeoperator\left(\bM^{(j,i)}\right) \\
&= - \convthreedtransposeoperator\left(\bM^{(j,i)}\right) \\
&= - \convthreedtransposeoperator\left(\bL^{(j,i)}\right) = \bL^{(i,j)}
\end{align*}
\textbf{Case 3:} For $i = j$, we further simplify equation \eqref{eq:M_diff_simple_3d}:
\begin{align}
&\bM^{(i,i)}_{s,t,u} - \bigg[\convthreedtransposeoperator\left(\bM^{(i,i)}\right)\bigg]_{s,t,u} \nonumber \\
&= \bM^{(i,i)}_{s,t,u} - \overline{\bM^{(i,i)}_{2p-s,2q-t,2r-u}} \label{eq:M_diff_simple_element_3d}
\end{align}
\textbf{Subcase 3(a):} For ($s \leq p-1$) or ($s = p,\ t \leq q-1$) or ($s = p,\ t = q,\ u \leq r-1$), we have:
$$ \bM^{(i,i)}_{2p-s,2q-t,2r-u} = 0 $$
Thus for ($s \leq p-1$) or ($s = p,\ t \leq q-1$) or ($s = p,\ t = q,\ u \leq r-1$): equation \eqref{eq:M_diff_simple_element_3d} simplifies to $\bM^{(i,i)}_{s,t,u}$. The result follows trivially from the very definition of $\bL^{(i,i)}_{s,t,u}$, i.e equation \eqref{eq:M_i_i_3d}.\\ \\
\textbf{Subcase 3(b):} For ($s \geq p+1$) or ($s = p ,\ t \geq q+1$) or ($s = p ,\ t = q,\ u \geq r+1$), we have: 
$$\bM^{(i,i)}_{s,t,u} = 0$$
Thus, equation \eqref{eq:M_diff_simple_element} simplifies to:
\begin{align*}
& \bM^{(i,i)}_{s,t,u} - \overline{\bM^{(i,i)}_{2p-s,2q-t,2r-u}} = - \overline{\bM^{(i,i)}_{2p-s,2q-t,2r-u}} 
\end{align*}
Since ($s \geq p+1$) or ($s = p,\ t \geq q+1$) or ($s = p,\ t = q,\ u \geq r+1$), we have: ($2p-s \leq p-1$) or ($2p-s=p,\ 2q - t \leq q-1$) or ($2p-s=p,\ 2q - t = q,\ 2u-r \leq r-1$) respectively. Thus using equation \eqref{eq:M_i_i_3d}, we have:
\begin{align*}
& - \overline{\bM^{(i,i)}_{2p-s,2q-t,2r-u}} = - \overline{\bL^{(i,i)}_{2p-s,2q-t,2r-u}} 
\end{align*}
Since $\bL^{(i,i)}$ is a skew-hermitian filter, we have from Theorem \ref{thm:appendix_3d_kernel}:
\begin{align*}
& \bL^{(i,i)}_{s,t,u} = - \overline{\bL^{(i,i)}_{2p-s,2q-t,2r-u}}
\end{align*}
Thus in this subcase, equation \eqref{eq:M_diff_simple_element_3d} simplifies to $\bL^{(i,i)}_{s,t,u}$ again.\\
\textbf{Subcase 3(c):} For $s = p ,\ t = q,\ u = r$, since $\bL^{(i,i)}$ is a skew-hermitian filter, we have:
\begin{align*}
& \bL^{(i,i)}_{p, q, r} = - \overline{\bL^{(i,i)}_{p,q,r}} \\
& \bL^{(i,i)}_{p, q, r} + \overline{\bL^{(i,i)}_{p,q,r}} = 0
\end{align*}
Thus, $\bL^{(i,i)}_{p, q, r}$ is a purely imaginary number. In this subcase
\begin{align*}
& \bM^{(i,i)}_{s,t,u} - \overline{\bM^{(i,i)}_{2p-s,2q-t,2r-u}} \\
& = \bM^{(i,i)}_{p,q,r} - \overline{\bM^{(i,i)}_{p,q,r}} = 2\bM^{(i,i)}_{p,q,r}
\end{align*}
Using equation \eqref{eq:M_i_i_3d} , we have:
\begin{align*}
2\bM^{(i,i)}_{p,q,r} = \bL^{(i,i)}_{p,q,r}
\end{align*}
Thus, we get:
\begin{align*}
&\bM^{(i,i)}_{p,q,r} - \bigg[\convthreedtransposeoperator\left(\bM^{(i,i)}\right)\bigg]_{p,q,r} = \bL^{(i,i)}_{p,q,r}
\end{align*}

Thus we have established: $\bL = \bM - \convthreedtransposeoperator(\bM)$. Note that the opposite direction of the if and only if statement follows trivially from the above proof.
\end{proof}

%\section{Figures}
%Invertible downsampling operation is demonstrated in Figure \ref{fig:downsampling}.

\section{$\mathrm{MaxMin}$ Activation function}\label{sec:appendix_maxmin}
Given a feature map $\bX \in \mathbb{R}^{2m \times n \times n}$ (we assume the number of channels in $\bX$ is a multiple of $2$), to apply the $\mathrm{MaxMin}$ activation function, we first divide the input into two chunks of equal size: $\bA$ and $\bB$ such that:
\begin{align*}
    & \bA = \bX_{:m,:,:}\\
    & \bB = \bX_{m:,:,:}
\end{align*}
Then the $\mathrm{MaxMin}$ activation function is given as follows:
\begin{align*}
    & \mathrm{MaxMin}(\bX)_{:m,:,:} = \max(\bA, \bB)\\
    & \mathrm{MaxMin}(\bX)_{m:,:,:} = \min(\bA, \bB)
\end{align*}


\end{document}

%% file: commands.tex
\usepackage{dsfont}
\usepackage{enumitem}
\usepackage[dvipsnames]{xcolor}
\usepackage{amsmath}
\usepackage{soul}
\usepackage{subcaption}
\usepackage{multirow}
\usepackage{booktabs}
\usepackage{csquotes}
\usepackage{graphicx}
\usepackage{mathtools}

\def\methodabv{SOC}
\def\methodname{Skew Orthogonal Convolution}
\def\methodnamebold{\textbf{S}kew \textbf{O}rthogonal \textbf{C}onvolution}

\def\methodabvcomplex{SUC}

\def\methodnamecomplexbold{\textbf{S}kew \textbf{U}nitary \textbf{C}onvolution}
\def\archname{LipConvnet}

\def\convtransposeoperator{\mathrm{conv\_transpose}}

\def\convthreedtransposeoperator{\mathrm{conv3d\_transpose}}

\usepackage{enumitem}

\setlist[description]{leftmargin=*,labelindent=*}
\setlist[itemize]{leftmargin=*}
\setlist[enumerate]{leftmargin=*}

%% file: example_paper.bbl
\begin{thebibliography}{42}
\providecommand{\natexlab}[1]{#1}
\providecommand{\url}[1]{\texttt{#1}}
\expandafter\ifx\csname urlstyle\endcsname\relax
  \providecommand{\doi}[1]{doi: #1}\else
  \providecommand{\doi}{doi: \begingroup \urlstyle{rm}\Url}\fi

\bibitem[Anil et~al.(2018)Anil, Lucas, and Grosse]{Anil2018SortingOL}
Anil, C., Lucas, J., and Grosse, R.~B.
\newblock Sorting out lipschitz function approximation.
\newblock In \emph{ICML}, 2018.

\bibitem[Arjovsky et~al.(2017)Arjovsky, Chintala, and
  Bottou]{arjovsky2017wasserstein}
Arjovsky, M., Chintala, S., and Bottou, L.
\newblock {W}asserstein generative adversarial networks.
\newblock In Precup, D. and Teh, Y.~W. (eds.), \emph{Proceedings of the 34th
  International Conference on Machine Learning}, volume~70 of \emph{Proceedings
  of Machine Learning Research}, pp.\  214--223, International Convention
  Centre, Sydney, Australia, 06--11 Aug 2017. PMLR.
\newblock URL \url{http://proceedings.mlr.press/v70/arjovsky17a.html}.

\bibitem[Bartlett et~al.(2017)Bartlett, Foster, and
  Telgarsky]{Bartlettgeneralization}
Bartlett, P.~L., Foster, D.~J., and Telgarsky, M.
\newblock Spectrally-normalized margin bounds for neural networks.
\newblock In \emph{Proceedings of the 31st International Conference on Neural
  Information Processing Systems}, NIPS'17, pp.\  6241--6250, USA, 2017. Curran
  Associates Inc.
\newblock ISBN 978-1-5108-6096-4.
\newblock URL \url{http://dl.acm.org/citation.cfm?id=3295222.3295372}.

\bibitem[Cao \& Gong(2017)Cao and Gong]{Cao2017MitigatingEA}
Cao, X. and Gong, N.~Z.
\newblock Mitigating evasion attacks to deep neural networks via region-based
  classification.
\newblock In \emph{Proceedings of the 33rd Annual Computer Security
  Applications Conference}, ACSAC 2017, pp.\  278–287, New York, NY, USA,
  2017. Association for Computing Machinery.
\newblock ISBN 9781450353458.
\newblock \doi{10.1145/3134600.3134606}.
\newblock URL \url{https://doi.org/10.1145/3134600.3134606}.

\bibitem[Ciss{\'{e}} et~al.(2017)Ciss{\'{e}}, Bojanowski, Grave, Dauphin, and
  Usunier]{cisseparseval2017}
Ciss{\'{e}}, M., Bojanowski, P., Grave, E., Dauphin, Y.~N., and Usunier, N.
\newblock Parseval networks: Improving robustness to adversarial examples.
\newblock In Precup, D. and Teh, Y.~W. (eds.), \emph{Proceedings of the 34th
  International Conference on Machine Learning, {ICML} 2017, Sydney, NSW,
  Australia, 6-11 August 2017}, volume~70 of \emph{Proceedings of Machine
  Learning Research}, pp.\  854--863. {PMLR}, 2017.
\newblock URL \url{http://proceedings.mlr.press/v70/cisse17a.html}.

\bibitem[Cohen et~al.(2019)Cohen, Rosenfeld, and Kolter]{Cohen2019CertifiedAR}
Cohen, J.~M., Rosenfeld, E., and Kolter, J.~Z.
\newblock Certified adversarial robustness via randomized smoothing.
\newblock In \emph{ICML}, 2019.

\bibitem[Croce et~al.(2019)Croce, Andriushchenko, and Hein]{croce2019provable}
Croce, F., Andriushchenko, M., and Hein, M.
\newblock Provable robustness of relu networks via maximization of linear
  regions.
\newblock \emph{AISTATS 2019}, 2019.

\bibitem[Edelman et~al.(1998)Edelman, Arias, and Smith]{Edelman1998TheGO}
Edelman, A., Arias, T.~A., and Smith, S.
\newblock The geometry of algorithms with orthogonality constraints.
\newblock \emph{SIAM J. Matrix Anal. Appl.}, 20:\penalty0 303--353, 1998.

\bibitem[Gouk et~al.(2020)Gouk, Frank, Pfahringer, and
  Cree]{gouk2020regularisation}
Gouk, H., Frank, E., Pfahringer, B., and Cree, M.~J.
\newblock Regularisation of neural networks by enforcing lipschitz continuity,
  2020.

\bibitem[Gulrajani et~al.(2017)Gulrajani, Ahmed, Arjovsky, Dumoulin, and
  Courville]{gulrajani2017improved}
Gulrajani, I., Ahmed, F., Arjovsky, M., Dumoulin, V., and Courville, A.~C.
\newblock Improved training of wasserstein gans.
\newblock In Guyon, I., Luxburg, U.~V., Bengio, S., Wallach, H., Fergus, R.,
  Vishwanathan, S., and Garnett, R. (eds.), \emph{Advances in Neural
  Information Processing Systems}, volume~30, pp.\  5767--5777. Curran
  Associates, Inc., 2017.
\newblock URL
  \url{https://proceedings.neurips.cc/paper/2017/file/892c3b1c6dccd52936e27cbd0ff683d6-Paper.pdf}.

\bibitem[Hoogeboom et~al.(2020)Hoogeboom, Satorras, Tomczak, and
  Welling]{Hoogeboom2020TheCE}
Hoogeboom, E., Satorras, V.~G., Tomczak, J., and Welling, M.
\newblock The convolution exponential and generalized sylvester flows.
\newblock \emph{ArXiv}, abs/2006.01910, 2020.

\bibitem[Jacobsen et~al.(2018)Jacobsen, Smeulders, and
  Oyallon]{jacobsen2018irevnet}
Jacobsen, J.-H., Smeulders, A.~W., and Oyallon, E.
\newblock i-revnet: Deep invertible networks.
\newblock In \emph{International Conference on Learning Representations}, 2018.
\newblock URL \url{https://openreview.net/forum?id=HJsjkMb0Z}.

\bibitem[Kumar et~al.(2020)Kumar, Levine, Goldstein, and
  Feizi]{cursedimensionalitykumar20}
Kumar, A., Levine, A., Goldstein, T., and Feizi, S.
\newblock Curse of dimensionality on randomized smoothing for certifiable
  robustness.
\newblock In III, H.~D. and Singh, A. (eds.), \emph{Proceedings of the 37th
  International Conference on Machine Learning}, volume 119 of
  \emph{Proceedings of Machine Learning Research}, pp.\  5458--5467. PMLR,
  13--18 Jul 2020.
\newblock URL \url{http://proceedings.mlr.press/v119/kumar20b.html}.

\bibitem[L{\'e}cuyer et~al.(2018)L{\'e}cuyer, Atlidakis, Geambasu, Hsu, and
  Jana]{Lcuyer2018CertifiedRT}
L{\'e}cuyer, M., Atlidakis, V., Geambasu, R., Hsu, D., and Jana, S. K.~K.
\newblock Certified robustness to adversarial examples with differential
  privacy.
\newblock In \emph{IEEE S\&P 2019}, 2018.

\bibitem[Levine et~al.(2019)Levine, Singla, and Feizi]{levine2019certifiably}
Levine, A., Singla, S., and Feizi, S.
\newblock Certifiably robust interpretation in deep learning, 2019.

\bibitem[Li et~al.(2019{\natexlab{a}})Li, Chen, Wang, and
  Carin]{Li2018CertifiedAR}
Li, B., Chen, C., Wang, W., and Carin, L.
\newblock Certified adversarial robustness with additive noise.
\newblock In Wallach, H., Larochelle, H., Beygelzimer, A., d'~Alch'e-Buc, F.,
  Fox, E., and Garnett, R. (eds.), \emph{Advances in Neural Information
  Processing Systems}, volume~32, pp.\  9464--9474. Curran Associates, Inc.,
  2019{\natexlab{a}}.
\newblock URL
  \url{https://proceedings.neurips.cc/paper/2019/file/335cd1b90bfa4ee70b39d08a4ae0cf2d-Paper.pdf}.

\bibitem[Li et~al.(2019{\natexlab{b}})Li, Haque, Anil, Lucas, Grosse, and
  Jacobsen]{li2019lconvnet}
Li, Q., Haque, S., Anil, C., Lucas, J., Grosse, R., and Jacobsen, J.-H.
\newblock Preventing gradient attenuation in lipschitz constrained
  convolutional networks.
\newblock \emph{Conference on Neural Information Processing Systems},
  2019{\natexlab{b}}.

\bibitem[Liu et~al.(2018)Liu, Cheng, Zhang, and Hsieh]{Liu2018TowardsRN}
Liu, X., Cheng, M., Zhang, H., and Hsieh, C.
\newblock Towards robust neural networks via random self-ensemble.
\newblock In \emph{ECCV}, 2018.

\bibitem[Long \& Sedghi(2020)Long and Sedghi]{long2019sizefree}
Long, P.~M. and Sedghi, H.
\newblock Generalization bounds for deep convolutional neural networks.
\newblock In \emph{International Conference on Learning Representations}, 2020.
\newblock URL \url{https://openreview.net/forum?id=r1e_FpNFDr}.

\bibitem[Madry et~al.(2018)Madry, Makelov, Schmidt, Tsipras, and
  Vladu]{madry2018towards}
Madry, A., Makelov, A., Schmidt, L., Tsipras, D., and Vladu, A.
\newblock Towards deep learning models resistant to adversarial attacks.
\newblock In \emph{International Conference on Learning Representations}, 2018.
\newblock URL \url{https://openreview.net/forum?id=rJzIBfZAb}.

\bibitem[Miyato et~al.(2018)Miyato, Kataoka, Koyama, and
  Yoshida]{miyato2018spectral}
Miyato, T., Kataoka, T., Koyama, M., and Yoshida, Y.
\newblock Spectral normalization for generative adversarial networks.
\newblock In \emph{International Conference on Learning Representations}, 2018.
\newblock URL \url{https://openreview.net/forum?id=B1QRgziT-}.

\bibitem[Peyré \& Cuturi(2018)Peyré and Cuturi]{peyre2018computational}
Peyré, G. and Cuturi, M.
\newblock Computational optimal transport, 2018.

\bibitem[Qian \& Wegman(2019)Qian and Wegman]{qian2018lnonexpansive}
Qian, H. and Wegman, M.~N.
\newblock L2-nonexpansive neural networks.
\newblock In \emph{International Conference on Learning Representations}, 2019.
\newblock URL \url{https://openreview.net/forum?id=ByxGSsR9FQ}.

\bibitem[Raghunathan et~al.(2018)Raghunathan, Steinhardt, and
  Liang]{Raghunathan2018SemidefiniteRF}
Raghunathan, A., Steinhardt, J., and Liang, P.
\newblock Semidefinite relaxations for certifying robustness to adversarial
  examples.
\newblock In \emph{NeurIPS}, 2018.

\bibitem[Ryu et~al.(2019)Ryu, Liu, Wang, Chen, Wang, and Yin]{realsn}
Ryu, E., Liu, J., Wang, S., Chen, X., Wang, Z., and Yin, W.
\newblock Plug-and-play methods provably converge with properly trained
  denoisers.
\newblock In Chaudhuri, K. and Salakhutdinov, R. (eds.), \emph{Proceedings of
  the 36th International Conference on Machine Learning}, volume~97 of
  \emph{Proceedings of Machine Learning Research}, pp.\  5546--5557, Long
  Beach, California, USA, 09--15 Jun 2019. PMLR.
\newblock URL \url{http://proceedings.mlr.press/v97/ryu19a.html}.

\bibitem[Salman et~al.(2019)Salman, Li, Razenshteyn, Zhang, Zhang, Bubeck, and
  Yang]{Salman2019ProvablyRD}
Salman, H., Li, J., Razenshteyn, I., Zhang, P., Zhang, H., Bubeck, S., and
  Yang, G.
\newblock Provably robust deep learning via adversarially trained smoothed
  classifiers.
\newblock In Wallach, H., Larochelle, H., Beygelzimer, A., d~Alch'e-Buc, F.,
  Fox, E., and Garnett, R. (eds.), \emph{Advances in Neural Information
  Processing Systems}, volume~32, pp.\  11292--11303. Curran Associates, Inc.,
  2019.
\newblock URL
  \url{https://proceedings.neurips.cc/paper/2019/file/3a24b25a7b092a252166a1641ae953e7-Paper.pdf}.

\bibitem[Sedghi et~al.(2019)Sedghi, Gupta, and Long]{sedghi2018singular}
Sedghi, H., Gupta, V., and Long, P.~M.
\newblock The singular values of convolutional layers.
\newblock In \emph{International Conference on Learning Representations}, 2019.
\newblock URL \url{https://openreview.net/forum?id=rJevYoA9Fm}.

\bibitem[Singh et~al.(2018)Singh, Gehr, Mirman, P{\"u}schel, and
  Vechev]{Singh2018FastAE}
Singh, G., Gehr, T., Mirman, M., P{\"u}schel, M., and Vechev, M.~T.
\newblock Fast and effective robustness certification.
\newblock In \emph{NeurIPS}, 2018.

\bibitem[Singla \& Feizi(2020)Singla and Feizi]{2020curvaturebased}
Singla, S. and Feizi, S.
\newblock Second-order provable defenses against adversarial attacks.
\newblock In III, H.~D. and Singh, A. (eds.), \emph{Proceedings of the 37th
  International Conference on Machine Learning}, volume 119 of
  \emph{Proceedings of Machine Learning Research}, pp.\  8981--8991. PMLR,
  13--18 Jul 2020.
\newblock URL \url{http://proceedings.mlr.press/v119/singla20a.html}.

\bibitem[Singla \& Feizi(2021)Singla and Feizi]{boundsingular}
Singla, S. and Feizi, S.
\newblock Fantastic four: Differentiable and efficient bounds on singular
  values of convolution layers.
\newblock In \emph{International Conference on Learning Representations}, 2021.
\newblock URL \url{https://openreview.net/forum?id=JCRblSgs34Z}.

\bibitem[Szegedy et~al.(2014)Szegedy, Zaremba, Sutskever, Bruna, Erhan,
  Goodfellow, and Fergus]{42503}
Szegedy, C., Zaremba, W., Sutskever, I., Bruna, J., Erhan, D., Goodfellow, I.,
  and Fergus, R.
\newblock Intriguing properties of neural networks.
\newblock In \emph{International Conference on Learning Representations}, 2014.
\newblock URL \url{http://arxiv.org/abs/1312.6199}.

\bibitem[Tolstikhin et~al.(2018)Tolstikhin, Bousquet, Gelly, and
  Schoelkopf]{tolstikhin2018wasserstein}
Tolstikhin, I., Bousquet, O., Gelly, S., and Schoelkopf, B.
\newblock Wasserstein auto-encoders.
\newblock In \emph{International Conference on Learning Representations}, 2018.
\newblock URL \url{https://openreview.net/forum?id=HkL7n1-0b}.

\bibitem[Tsipras et~al.(2018)Tsipras, Santurkar, Engstrom, Turner, and
  Madry]{tsipras2019robustness}
Tsipras, D., Santurkar, S., Engstrom, L., Turner, A., and Madry, A.
\newblock Robustness may be at odds with accuracy.
\newblock In \emph{ICLR}, 2018.

\bibitem[Tsuzuku et~al.(2018)Tsuzuku, Sato, and
  Sugiyama]{Tsuzuku2018LipschitzMarginTS}
Tsuzuku, Y., Sato, I., and Sugiyama, M.
\newblock Lipschitz-margin training: Scalable certification of perturbation
  invariance for deep neural networks.
\newblock In \emph{NeurIPS}, 2018.

\bibitem[Villani(2008)]{villani2012optimal}
Villani, C.
\newblock Optimal transport, old and new, 2008.

\bibitem[Weng et~al.(2018)Weng, Zhang, Chen, Song, Hsieh, Boning, and
  Daniel]{weng2018CertifiedRobustness}
Weng, T.-W., Zhang, H., Chen, H., Song, Z., Hsieh, C.-J., Boning, D., and
  Daniel, I. S. D.~A.
\newblock Towards fast computation of certified robustness for relu networks.
\newblock In \emph{International Conference on Machine Learning (ICML)}, july
  2018.

\bibitem[Wong \& Kolter(2018)Wong and Kolter]{Wong2017ProvableDA}
Wong, E. and Kolter, Z.
\newblock Provable defenses against adversarial examples via the convex outer
  adversarial polytope.
\newblock In Dy, J. and Krause, A. (eds.), \emph{Proceedings of the 35th
  International Conference on Machine Learning}, volume~80 of \emph{Proceedings
  of Machine Learning Research}, pp.\  5286--5295, Stockholmsmässan, Stockholm
  Sweden, 10--15 Jul 2018. PMLR.
\newblock URL \url{http://proceedings.mlr.press/v80/wong18a.html}.

\bibitem[Wong et~al.(2018)Wong, Schmidt, Metzen, and Kolter]{Wong2018ScalingPA}
Wong, E., Schmidt, F.~R., Metzen, J.~H., and Kolter, J.~Z.
\newblock Scaling provable adversarial defenses.
\newblock In \emph{NeurIPS}, 2018.

\bibitem[Wong et~al.(2020)Wong, Rice, and Kolter]{Wong2020Fast}
Wong, E., Rice, L., and Kolter, J.~Z.
\newblock Fast is better than free: Revisiting adversarial training.
\newblock In \emph{International Conference on Learning Representations}, 2020.
\newblock URL \url{https://openreview.net/forum?id=BJx040EFvH}.

\bibitem[Xiao et~al.(2018)Xiao, Bahri, Sohl-Dickstein, Schoenholz, and
  Pennington]{xiao2018dynamical}
Xiao, L., Bahri, Y., Sohl-Dickstein, J., Schoenholz, S., and Pennington, J.
\newblock Dynamical isometry and a mean field theory of {CNN}s: How to train
  10,000-layer vanilla convolutional neural networks.
\newblock In Dy, J. and Krause, A. (eds.), \emph{Proceedings of the 35th
  International Conference on Machine Learning}, volume~80 of \emph{Proceedings
  of Machine Learning Research}, pp.\  5393--5402, Stockholmsmässan, Stockholm
  Sweden, 10--15 Jul 2018. PMLR.
\newblock URL \url{http://proceedings.mlr.press/v80/xiao18a.html}.

\bibitem[Zhang et~al.(2018)Zhang, Weng, Chen, Hsieh, and
  Daniel]{zhang2018crown}
Zhang, H., Weng, T.-W., Chen, P.-Y., Hsieh, C.-J., and Daniel, L.
\newblock Efficient neural network robustness certification with general
  activation functions.
\newblock In \emph{Advances in Neural Information Processing Systems (NIPS),
  arXiv preprint arXiv:1811.00866}, dec 2018.

\bibitem[Zhang et~al.(2019)Zhang, Zhang, and Hsieh]{zhang2018recurjac}
Zhang, H., Zhang, P., and Hsieh, C.-J.
\newblock Recurjac: An efficient recursive algorithm for bounding jacobian
  matrix of neural networks and its applications.
\newblock In \emph{AAAI Conference on Artificial Intelligence (AAAI), arXiv
  preprint arXiv:1810.11783}, dec 2019.

\end{thebibliography}
